\documentclass{article}
\usepackage{graphicx}
\usepackage{float}
\usepackage{caption}
\usepackage{subcaption} 
\usepackage[super]{nth}
\usepackage{amsmath}
\usepackage{bm}
\usepackage{hyperref}
\usepackage[dvipsnames]{xcolor}  
\usepackage{lineno} 
\usepackage{placeins}
\usepackage{array}
\usepackage{tabu}
\usepackage{amssymb}
\usepackage{mathtools}
\usepackage{lineno}
\usepackage{wrapfig}
\usepackage{booktabs}       
\usepackage{float}
\usepackage{authblk}
\usepackage{enumerate}
\usepackage{xcolor}
\usepackage{multirow}
\usepackage{enumitem}
\usepackage{amsthm}
\usepackage{amsmath}
\usepackage{comment}
\usepackage{booktabs} 
\usepackage[accepted]{icml2021}

\graphicspath{Plots}

\newtheorem{lemma}{Lemma}

\newtheorem{thm}{Theorem}[section]

\DeclareMathOperator*{\argmax}{arg\,max}

\def\R {\mathbb{R}}
\def\E {\mathbb{E}}

\def\Z {\mathbb{Z}}

\def\cN {\mathcal{N}}

\def\SglRm{S_{h,l}^m}
\def\SglponeRm{S_{h,l+1}^m}
\def\SglRtildemi{S_{h,l}^{\tilde{m}_{i0}}}
\def\SgoneRktildem{S_{h,1}^{k\tilde{m}}}

\def\SglRnm{S_{h,l}^{n,m}}
\def\SgljRnjmj{S_{h,l_j}^{n_j,m_{j}}}
\def\SglMn{S_{h,l}^{\mathcal{P}}}
\def\SgkMnKone{S_{h,1}^{\mathcal{P}}}
\def\tightset{\Omega}

\def\radonj{\mathcal{P}(\R^{n_j})}
\def\radonone{\mathcal{P}(\R^{n_1})}
\def\radonk{\mathcal{P}(\R^{n_k})}

\def\fabsnn{\tilde{H}}
\def\radonKj{\mathcal{P}(K_j)}
\def\radonK{\mathcal{P}(K)}

\icmltitlerunning{Measure-conditional Discriminator with Stationary Optimum for GANs and Statistical Distance Surrogates}

\begin{document}

\twocolumn[

\icmltitle{Measure-conditional Discriminator with Stationary Optimum \\for GANs and Statistical Distance Surrogates}

\begin{icmlauthorlist}
\icmlauthor{Liu Yang}{br}
\icmlauthor{Tingwei Meng}{br}
\icmlauthor{George Em Karniadakis}{br}
\end{icmlauthorlist}

\icmlaffiliation{br}{Division of Applied Mathematics, Brown University, Providence, Rhode Island, USA}

\icmlcorrespondingauthor{George Em Karniadakis}{george\_karniadakis@brown.edu}

\icmlkeywords{Machine Learning, ICML}

\vskip 0.3in 
]

\printAffiliationsAndNotice{}

\begin{abstract}
    We propose a simple but effective modification of the discriminators, namely measure-conditional discriminators, as a plug-and-play module for different GANs. By taking the generated distributions as part of input so that the target optimum for the discriminator is stationary, the proposed discriminator is more robust than the vanilla one. A variant of the measure-conditional discriminator can also handle multiple target distributions, or act as a surrogate model of statistical distances such as KL divergence with applications to transfer learning.

\end{abstract}

\section{Introduction}

Generative adversarial networks (GANs)~\cite{goodfellow2014generative} have proven to be successful in training generative models to fit the target distributions. Apart from tasks of image generation~\cite{brock2018large, zhu2017unpaired}, text generation~\cite{pmlr-v70-zhang17b,fedus2018maskgan}, etc., GANs have also been used in physical problems to infer unknown parameters in stochastic systems~\cite{yang2020physics, yang2019adversarial, yang2020generative}. Due to the variety of the generative models, GAN loss functions, and the need for high accuracy inferences, such tasks usually set a strict requirement to the robustness of GANs, as well as the similarity between the generated and target distributions in various metrics. 

The optimum for the discriminator is, in general, non-stationary, i.e., it varies during the training, since it depends on the generated distributions. Such issue could lead to instability or oscillation in the training. Here, we propose a simple but effective modification to the discriminator as a plug-and-play module for different GANs, including vanilla GANs~\cite{goodfellow2014generative}, Wasserstein GANs with gradient penalty (WGAN-GP)~\cite{gulrajani2017improved}, etc. The main idea is to make the discriminator conditioned on the generated distributions, so that its optimum is stationary during the training.

The neural network architecture of the
measure-conditional discriminator is adapted from DeepSets neural network~\cite{zaheer2017deep}, which is widely used in point cloud related tasks. It is also used in GANs~\cite{li2018point} where each sample corresponds to a point cloud, while we target on more general tasks where each sample corresponds to a particle, an image, etc.
In \citet{lucas2018mixed}, the discriminator takes the mixture of real and generated distributions (instead of individual samples) as input, but it also has a non-stationary target optimum, and performs worse than our discriminators in experiments. We also emphasize the difference between the measure-conditional discriminator and conditional GANs~\cite{mirza2014conditional}. Conditioned on a vector featuring the target distributions, the conditional GANs still have a non-stationary target optimum for the discriminator and are limited to the scenarios where the samples can be categorized. Moreover, the measure-conditional discriminator can also be applied in conditional GANs, by making the original discriminator further conditioned on the generated distributions.

In Section~\ref{sec:Stationary} we discuss why we need stationary target optimum for the discriminator. In Section~\ref{sec:MCDis} we give a detailed description of the proposed discriminator neural networks and how to apply them in GANs. In Section~\ref{sec:StatDist} we extend the application of measure-conditional discriminators as surrogate models of statistical distances. In Section~\ref{sec:UA} we present a universal approximation theorem of the neural networks used in this paper. The experimental results are shown in Section~\ref{sec:results}. We conclude in Section~\ref{sec:Summary}.

\begin{figure*}[ht]
     \centering
     \begin{subfigure}[b]{0.19\textwidth}
         \centering
         \includegraphics[width=\textwidth]{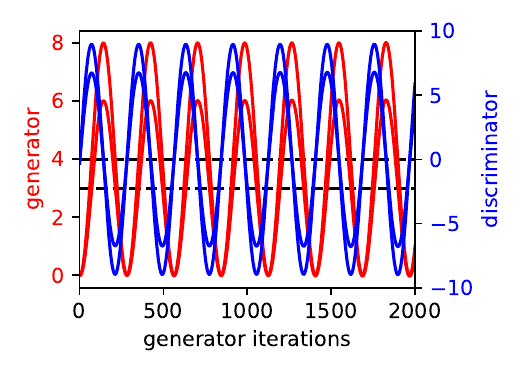}
         \caption{}
         \label{fig:ToyVanilla}
     \end{subfigure}
     \begin{subfigure}[b]{0.19\textwidth}
         \centering
         \includegraphics[width=\textwidth]{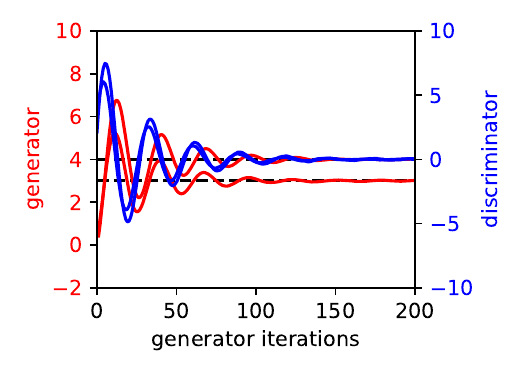}
         \caption{}
         \label{fig:ToyOMD01}
     \end{subfigure}
     \begin{subfigure}[b]{0.19\textwidth}
         \centering
         \includegraphics[width=\textwidth]{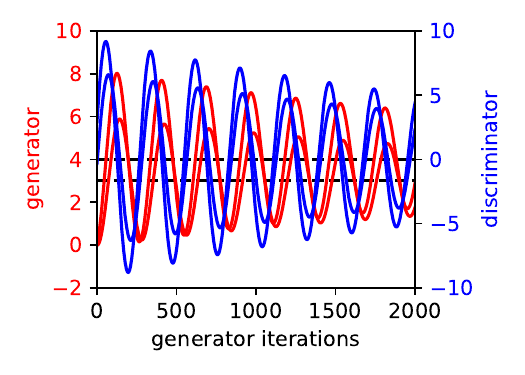}
         \caption{}
         \label{fig:ToyOMD001}
     \end{subfigure}
     \begin{subfigure}[b]{0.19\textwidth}
         \centering
         \includegraphics[width=\textwidth]{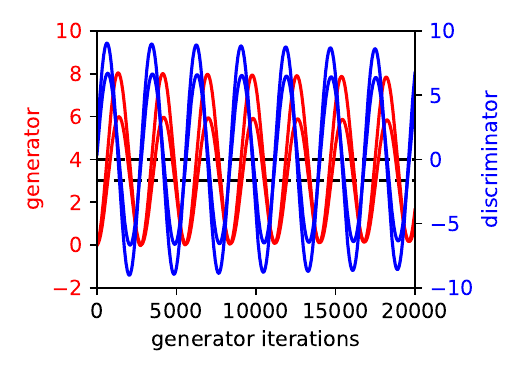}
         \caption{}
         \label{fig:ToyOMD0001}
     \end{subfigure}
    \begin{subfigure}[b]{0.19\textwidth}
         \centering
         \includegraphics[width=\textwidth]{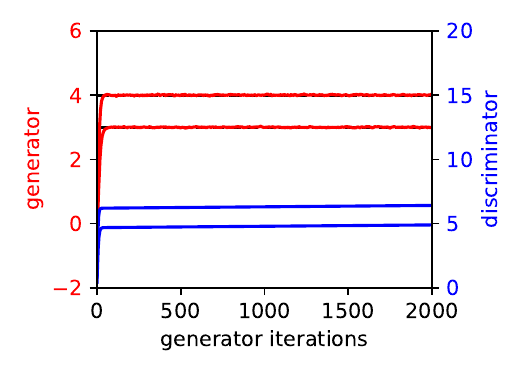}
         \caption{}
         \label{fig:ToyDistribution}
     \end{subfigure}
        \caption{Results for the illustrative example of oscillation. (a): vanilla discriminator using gradient descent with learning rate 0.01. (b-d): vanilla discriminator using optimistic mirror descent, with learning rate 0.1, 0.01, and 0.001, respectively. Note that the multiplications of learning rate and training iterations are kept the same. (e): measure-conditional discriminator using gradient descent with learning rate 0.01. The red and blue lines show the generator and discriminator parameters in both two dimensions, while the black horizontal lines represent the ground truth for the generator parameters.}
        \label{fig:ToyGAN}
\end{figure*}

\section{Stationary Target Optimum for the Discriminator}\label{sec:Stationary}

In general, there are two mathematical perspectives for GANs. The first perspective is to view GANs as a two-player zero-sum game between the generator $G$ and the discriminator $D$. The hope is that the iterative adversarial training of $G$ and $D$ will lead to the Nash equilibrium of this zero-sum game, where the generated distribution will be identical to the target distribution of real data. The second perspective is that the discriminator gives the distance between the generated distribution and the target distribution in a variational form. For example, vanilla GANs can be formulated as:
\begin{equation}\label{eqn:vanillaGAN}
\begin{aligned}
    \min_G & \max_D V(G_{\#}\cN,D),\\
    V(G_{\#}\cN,D) &= \mathbb{E}_{x\sim G_{\#}\cN}\log(1-D(x)) + \mathbb{E}_{x\sim Q}\log(D(x))\\
    & = \mathbb{E}_{z\sim\cN}\log(1-D(G(z))) + \mathbb{E}_{x\sim Q}\log(D(x)),
\end{aligned}
\end{equation}
while Wasserstein GANs (WGANs) can be formulated as:
\begin{equation}\label{eqn:WGAN}
\begin{aligned}
    \min_G & \max_{D \text{ is 1-Lipschitz}} V(G_{\#}\cN,D),\\
     V(G_{\#}\cN,D) & = -\mathbb{E}_{x\sim G_{\#}\cN}D(x) + \mathbb{E}_{x\sim Q}D(x)\\
    &= -\mathbb{E}_{z\sim\cN}D(G(z)) + \mathbb{E}_{x\sim Q}D(x),
\end{aligned}
\end{equation}
where $Q$ represents the target distribution, $\cN$ represents input noise, $_{\#}$ is the push-forward operator, thus $G_{\#}\cN$ represents the generated distribution. In vanilla GANs and WGANs, $\max_D V(G_{\#}\cN,D)$ and $\max_{D\text{ is 1-Lipschitz}} V(G_{\#}\cN,D)$ are nothing but the Jensen-Shannon (JS) divergence and the Wasserstein-1 distance up to constants between $G_{\#}\cN$ and $Q$, respectively.

\subsection{Non-stationary Target Optimum Hurts: An Illustrative Example}

From both two perspectives of GANs, the discriminator will approach its optimum $D^*$ in each iteration.
However, we will use the following illustrative example to demonstrate that $D^*$ could be totally different as we perturb the generator, and such issue would lead to the oscillation of both generator and discriminator during training.
This illustrative problem is adapted from \citet{daskalakis2018training} with different analysis. We first consider a linear discriminator as well as a translation function as the generator, i.e.,
\begin{equation}\label{eqn:ill_WGAN_v}
\begin{aligned}
    D_w(x) &= \langle w, x \rangle = \sum_{i=1}^d w_ix_i ,\\
    G_\theta(z) &= z + \theta, z \sim \cN(0, I)=:\cN
\end{aligned}
\end{equation}
with the target real distribution $Q = \cN(v,I)$. The goal is to learn $\theta$ with ground truth $\theta^* = v$. 

The WGAN with weight-clipping is formulated as 
\begin{equation}\label{eqn:ill_minmax}
\begin{aligned}
    \min_\theta & \max_{|w_i| \le c} -\mathbb{E}_{z\sim\cN}D(G(z))) + \mathbb{E}_{x\sim Q}D(x),
\end{aligned}
\end{equation}
where $c>0$ is the weight-clipping bound. In practice, we will use the empirical distributions to calculate the expectations, but if we calculate it analytically, we have the following min-max formulation:
\begin{equation}\label{eqn:ill_analytic}
\begin{aligned}
    \min_\theta & \max_{|w_i| \le c} \langle w, v- \theta\rangle.
\end{aligned}
\end{equation}
This two-player game has a unique equilibrium at $\theta = v, w=0$, which appears to be satisfactory. However, if we set $\theta = v +\epsilon$, where $\epsilon \neq0$ is the inevitable small fluctuation vector due to the randomness of the training data, moments in the optimizer, etc., then $w$ would achieve the corresponding optimum at $-\text{sgn}(\epsilon)c$, where ``sgn'' denotes the component-wise sign function. In other words, the optimal $w$ would jump between $c$ and $-c$ for each entry as $\theta$ fluctuates around the ground truth. Such issue of jumping optimum will lead to the oscillation of both the generator and discriminator during training, as is illustrated in Figure \ref{fig:ToyVanilla}, where we test on a 2D problem with ground truth $v = (3,4)$ and $c = 10$. 
Note that even if we set the discriminator as a general 1-Lipschitz function as in Equation~\ref{eqn:WGAN}, the corresponding optimum will be $D^*(x)=-\epsilon \cdot x/|\epsilon| + d$, which is still sensitive to the small fluctuation $\epsilon$, where $d$ is an arbitrary constant.

To remove the oscillation, \citet{daskalakis2018training} proposed to replace the gradient descent (GD) for the min-max formulation~(\ref{eqn:ill_analytic})
\begin{equation}
\begin{aligned}
    w_{t+1}-w_t &=  \eta(v-\theta_t)=: \eta \dot{w}_t,\\
    \theta_{t+1} - \theta_t &= \eta w_t =:\eta \dot{\theta}_t,
\end{aligned}
\end{equation}
with the optimistic mirror descent (OMD)
\begin{equation}
\begin{aligned}
    w_{t+1}-w_t & =2\eta(v-\theta_t) - \eta(v-\theta_{t-1}) = \eta \dot{w}_t - \eta^2 \dot{\theta}_{t-1},\\
    \theta_{t+1} - \theta_t &=2\eta w_t-\eta w_{t-1} = \eta \dot{\theta}_t + \eta^2 \dot{w}_{t-1},
\end{aligned}
\end{equation}
where $\eta$ is the learning rate. However, we report that the oscillation decay with OMD could be very slow when a small learning rate is used, as is illustrated in Figures~\ref{fig:ToyOMD01},\ref{fig:ToyOMD001},\ref{fig:ToyOMD0001}.
This is because the differences between the OMD and GD update, $-\eta^2 \dot{\theta}_{t-1}$ and $\eta^2 \dot{w}_{t-1}$, are second order w.r.t. $\eta$, i.e., one order higher than the GD update, $ \eta \dot{w}_t$ and $\eta \dot{\theta}_t$. The difference between the GD and OMD dynamics thus vanishes as $\eta$ goes to zero.

In the following, we will propose a much simpler and more effective strategy to remove these oscillations.

\subsection{Benefits of Stationary Target Optimum}

Since the aforementioned problem is due to the fact that the target optimum for the discriminator is non-stationary, to remove the oscillation, we propose  to modify the discriminator architecture so that its target optimum is stationary during the training.

While keeping the generator and min-max formulation unchanged as in Equations~\ref{eqn:ill_WGAN_v} and \ref{eqn:ill_minmax}, we set the discriminator
\begin{equation}\label{eqn:ill_weightdis}
\begin{aligned}
    D_w(x) =\sum_{i=1}^d w_ix_i (\mathbb{E}_{x\sim Q}(x) - \mathbb{E}_{z\sim \cN}(G_\theta(z)))_i 
\end{aligned}
\end{equation}
where $(\cdot)_i$ denotes the $i$-th component. The only differences between Equations~\ref{eqn:ill_WGAN_v} and~\ref{eqn:ill_weightdis} are the weights for $w_ix_i$.
If we calculate the expectations in the min-max formulation \ref{eqn:ill_minmax} analytically, we will have
\begin{equation}
\begin{aligned}
    \min_\theta & \max_{|w_i| \le c} \sum_{i=1}^d w_i(v_i - \theta_i)^2.
\end{aligned}
\end{equation}

For this min-max problem, any $w$ with non-negative entries and $\theta = v$ is a Nash equilibrium. If we set $\theta = v +\epsilon$ with $\epsilon \neq 0$, then $w$ would achieve the corresponding optimum at $c$ for each entry, i.e., the target optimum for the discriminator is stationary. As shown in Figure~\ref{fig:ToyDistribution}, the oscillation is totally removed. Each entry of $w$ is heading to the optimum $c$ in the early stage of training, while the change becomes negligible after $\theta$ converges to $v$, indicating that the Nash equilibrium is achieved.

The magic of the above solution lies in the fact that by designing the discriminator properly, we have a stationary target optimum for the discriminator during the training. Is this possible for more general GAN tasks where generators and discriminators are neural networks, and the target distributions are more flexible?

Note that the discriminator in Equation~\ref{eqn:ill_weightdis} can be interpreted as a discriminator conditioned on the generated and target distribution, so for more general GAN tasks we can simply design the discriminator as
\begin{equation}\label{eqn:MeasCondD}
    D_{mc} = D_{mc}(x,G_{\#}\cN),
\end{equation}
where $G_{\#}\cN$ is the generated distribution. The target distribution $Q$ is omitted in the input since it is usually fixed in a GAN task, but we will revisit this in Section~\ref{sec:StatDist}. We name the discriminator in Equation~\ref{eqn:MeasCondD} as a ``\textit{measure-conditional discriminator}'' since it is conditioned on the probability measure corresponding to the generated distribution.
The proposed measure-conditional discriminator can be a plug-and-play module in a variety of GANs. We only need to replace the original discriminator $D(\cdot)$ with $D_{mc}(\cdot, G_{\#}\cN)$, while the generator and the min-max formulation of GANs will be kept unchanged. A more detailed introduction of the measure-conditional discriminator in GANs will be presented in Section~\ref{sec:MCDis}.

We can see that by taking $G_{\#}\cN$ as part of the input, the measure-conditional discriminator will have a stationary target optimum during the training process. Indeed, for a general GAN problem originally formulated as
\begin{equation}\label{eqn:gan_minmax}
   \min_{G\in \mathcal{G}} \max_{D\in \mathcal{D}} V(G_{\#}\cN,D),
\end{equation}
with two examples given in Equation~\ref{eqn:vanillaGAN} and Equation~\ref{eqn:WGAN}, the target optimum for the measure-conditional discriminator is
\begin{equation}\label{eqn:targetoptimum}
   D^*_{mc}(x,G_{\#}\cN) = \left(\argmax_{D\in \mathcal{D}}  V(G_{\#}\cN,D)\right)(x).
\end{equation}

Although $G_{\#}\cN$ varies during the training, $ D^*_{mc}$ is a function of $G_{\#}\cN$ and $x$ is stationary.

From the perspective of statistical distances, the target optimum $D^*_{mc}$ is exactly a surrogate model for the distance between $G_{\#}\cN$ and $Q$. For example, in vanilla GANs, 
\begin{equation}\label{eqn:gansurrogate}
   \mathbb{E}_{z\sim\cN}\log(1-D^*_{mc}(G(z), G_{\#}\cN)) + \mathbb{E}_{x\sim Q}\log(D^*_{mc}(x,G_{\#}\cN))
\end{equation}
represents the JS divergence between $G_{\#}\cN$ and $Q$ up to constants, while in WGANs,
\begin{equation}\label{eqn:wgansurrogate}
   -\mathbb{E}_{z\sim\cN}D^*_{mc}(G(z), G_{\#}\cN) + \mathbb{E}_{x\sim Q}D^*_{mc}(x, G_{\#}\cN)
\end{equation}
represents the Wasserstein-1 distance between $G_{\#}\cN$ and $Q$ up to constants. 

It is hard to attain the target optimum $D^*_{mc}$, considering that it is a function of measures. However, we note that $D_{mc}$ does not need to attain $D^*_{mc}$ for the convergence of GANs. Indeed, we only require $D_{mc}$ to approximate the optimum for $G_{\#}\cN$, instead of the whole space of probability measures.

The vanilla discriminator only utilizes the result of the previous \textit{one} iteration to provide the initialization. If the optimum is sensitive to the generated distribution as in the above illustrative example, in each iteration, the vanilla discriminator need to ``forget the wrong optimum'' inherited from the previous iteration and head for the new one in a few discriminator updates. In contrast, the measure-conditional discriminator progressively head for the stationary target optimum in \textit{all} the iterations. In fact, even outdated generated distributions can be used to train the measure-conditional discriminator. If the optimum is sensitive to the generated distribution, the measure-conditional discriminator does not need to forget the inheritances from previous iterations, but
only need to learn the sensitivity w.r.t. the input measure.

To some extent, the generator and the measure-conditional discriminator are trained in a \textit{collaborative} way, in that the generator adaptively produces new distributions as training data to help the discriminator approximate $D^*_{mc}$, while the discriminator provides statistical distances to help the generator approach the target distribution. This concept is actually similar to reinforcement learning in the actor-critic framework~\cite{grondman2012survey}, with parallelism between the generator and actor, as well as between the discriminator and critic.

\section{Measure-conditional Discriminator in GANs}\label{sec:MCDis}

Proposed in \citet{zaheer2017deep}, the DeepSets neural network having the form of $H(X) = g(\sum_{x_i\in X}f(x_i))$
is widely used to represent a function of a point cloud $X$. The summation can be replaced by averaging to represent a function of probability measure $P$, i.e., $H(P) = g(\frac{1}{n}\sum_{i=1}^nf(x_i))$
where $\{x_i\}_{i=1}^n$ are samples from $P$.

In order to take a probability measure and an individual sample simultaneously as the discriminator input, we adapt the neural network architecture above to get
\begin{equation}\label{eqn:DxP}
    \begin{aligned}
        D_{mc}(x, P) = h(\mathbb{E}_{y\sim P}[f(y)], g(x))
        \approx h(\frac{1}{n}\sum_{i=1}^{n}f(y_{i}), g(x))
    \end{aligned}
\end{equation}
where $f$, $g$ and $h$ are neural networks, and $\{y_{i}\}_{i=1}^{n}$ are samples from $P$.

The measure-conditional discriminator is a plug-and-play module in a various GANs. The only modification is to replace $D(\cdot)$ with $D_{mc}(\cdot, G_{\#}\cN)$. For example, in vanilla GANs, the loss functions for the generator and the discriminator are
\begin{equation}
    \begin{aligned}
        L_g =& \mathbb{E}_{z\sim\cN}\log(1-D_{mc}(G(z),G_{\#}\cN)) \\&+ \mathbb{E}_{x\sim Q}\log(D_{mc}(x,G_{\#}\cN)), \\
        L_d =& - \mathbb{E}_{z\sim\cN}\log(1-D_{mc}(G(z),G_{\#}\cN)) \\&- \mathbb{E}_{x\sim Q}\log(D_{mc}(x,G_{\#}\cN)),
    \end{aligned}
\end{equation}
respectively. In WGAN with gradient penalty (WGAN-GP), the loss functions are 
\begin{equation}
    \begin{aligned}
        L_g =& -\mathbb{E}_{z\sim\cN}D_{mc}(G(z), G_{\#}\cN) + \mathbb{E}_{x\sim Q}D_{mc}(x,G_{\#}\cN),\\
        L_d =& \mathbb{E}_{z\sim\cN}D_{mc}(G(z), G_{\#}\cN) - \mathbb{E}_{x\sim Q}D_{mc}(x,G_{\#}\cN) \\&+ \lambda \mathbb{E}_{\hat{x}\sim \rho_{\hat{x}} } (\Vert\nabla_{\hat{x}} D_{mc}(\hat{x}, G_{\#}\cN) \Vert_2 - 1)^2,
    \end{aligned}
\end{equation}
respectively, where $\lambda$ is the weight for gradient penalty, and $\rho_{\hat{x}}$ is the distribution generated by sampling uniformly on interpolation lines between pairs of points sampled from real distributions and generated distributions. Note that the expectation over the real distribution $Q$ cannot be removed from the generator loss, since this term dependents on the generator now.

\section{Measure-conditional Discriminator for Statistical Distances Surrogate}\label{sec:StatDist}

The target distribution is usually fixed in GANs, thus omitted in the input of $D_{mc}$. Taking one step further, we will build a measure-conditional discriminator $D_{sr}$ conditioned on two probability measures $P$ and $Q$, which can act as a surrogate model to approximate the statistical distances between $P$ and $Q$.
Specifically, the neural network $D_{sr}$ is formulated as 
\begin{equation}\label{eqn:DxPQ}
    \begin{aligned}
        D_{sr}(x, P, Q) &= h(\mathbb{E}_{y\sim P}[f_1(y)], \mathbb{E}_{y\sim Q}[f_2(y)], g(x))\\
        &\approx h(\frac{1}{n_1}\sum_{i=1}^{n_1}f_1(y_{i}^{P}), \frac{1}{n_2}\sum_{i=1}^{n_2}f_2(y_{i}^{Q}), g(x))
    \end{aligned}
\end{equation}
where $f_1$, $f_2$, $g$ and $h$ are neural networks, and $\{y_{i}^{P}\}_{i=1}^{n_1}$ and $\{y_{i}^{Q}\}_{i=1}^{n_2}$ are samples from $P$ and $Q$, respectively. 

\subsection{Unsupervised Training with Variational Formula}
We will train $D_{sr}$ using the variational form of the statistical distances, in the same spirit as in GANs. Here, we take the KL divergence as an example, which has the following variational formula \cite{nguyen2010estimating}:
\begin{equation}\label{eqn:KLvariational}
\begin{aligned}
    D_{KL}(P||Q) &= \sup_{g>0} \left( \mathbb{E}_{x\sim P} [\log(g(x))] - \mathbb{E}_{x \sim Q} [g(x)]  + 1\right),\\
    &= \sup \left( \mathbb{E}_{x\sim P} [g(x)] - \mathbb{E}_{x \sim Q} [\exp(g(x))]  + 1\right).
\end{aligned}
\end{equation}
Thus, the loss function for $D_{sr}$ can be written as
\begin{equation}\label{eqn:KLtrain}
\begin{aligned}
    L_{KL} &= \mathbb{E}_{(P, Q)\sim \mu}[-l_{KL}(P,Q)] \\
    l_{KL}(P,Q) & = \mathbb{E}_{x\sim P} [D_{sr}(x, P, Q)] \\ &- \mathbb{E}_{x \sim Q} [\exp(D_{sr}(x, P, Q))] + 1, \\
\end{aligned}
\end{equation}
where $\mu$ represents the distribution for the probability measure pairs $(P,Q)$ in the training. Ideally, $l_{KL}(P,Q)$ will approximate $D_{KL}(P,Q)$ if $D_{sr}$ achieves optimum. Similar loss functions can be constructed for many other statistical distances like JS divergence, total variation etc., provided with variational forms as in Equation~\ref{eqn:KLvariational}. We also give an example of a surrogate model with results for the optimal transport map in Supplementary Material.

Note that after the optimization of $D_{sr}$ (which can be offline), via a forward propagation of $D_{sr}$, we can estimate $l_{KL}(P,Q)$ as an approximation of $D_{KL}(P, Q)$ for various $(P,Q)$ pairs sampled from $\mu$, and even for $(P,Q)$ pairs that are never seen in the training procedure (thanks to the generalization of neural networks). The computational cost for the forward propagation grows linearly w.r.t. the sample size.
More importantly, no labels are required to train $D_{sr}$. Instead, we only need to prepare samples of $P$ and $Q$ as training data. 

Here, $\mu$ can, of course, be prescribed by the users. It can also be decided actively during the training, depending on specific tasks. In the context of GANs, with $Q$ being different target distributions and $P$ being the corresponding generated distributions, $(P,Q)$ samples can be induced from a family of GAN tasks. $(P,Q)$ samples can also be induced from a single GAN task, if we need to fit multiple target distributions simultaneously, e.g., the distributions at multiple time instants in time-dependent problems. We will demonstrate this with an example in Section~\ref{sec:results}.

\subsection{Transfer Learning with Statistical Distance Surrogate}

As a surrogate model of statistical distances between distributions, it is possible that a well-trained $D_{sr}$ can be transferred to GANs and act as a discriminator without any update.
For example, if $D_{sr}$ is pretrained with Equation~\ref{eqn:KLtrain}, the generator can be trained with the loss function $L_g= l_{KL}(G_{\#}\cN, Q)$, with $l_{KL}$ from Equation~\ref{eqn:KLtrain}.
However, training $G$ with a frozen $D_{sr}$ requires that $(G_{\#}\cN, Q)$ is not an outlier of $\mu$. This typically means that the degree of freedom for $G$ is limited.

Alternatively, the pretrained $D_{sr}$ can be employed as an initialization of the discriminator and be fine-tuned in GANs. With $Q$ fixed as the target distribution, $D_{sr}$ is reduced to a function of $P$ and $x$, just as $D_{mc}$. We can then train it iteratively with the generator as in Section~\ref{sec:MCDis} . Note that the loss function in GANs should coincide with that in the pretraining of $D_{sr}$. For example, if $D_{sr}$ is pretrained with Equation~\ref{eqn:KLtrain}, then the loss functions for the generator $G$ is given by $L_g= l_{KL}(G_{\#}\cN, Q)$, while the loss functions for $D_{sr}$ is $-L_g$.

\begin{figure*}[ht]
     \centering
         \centering
        \begin{subfigure}[b]{0.147\textwidth}
         \includegraphics[width=\textwidth]{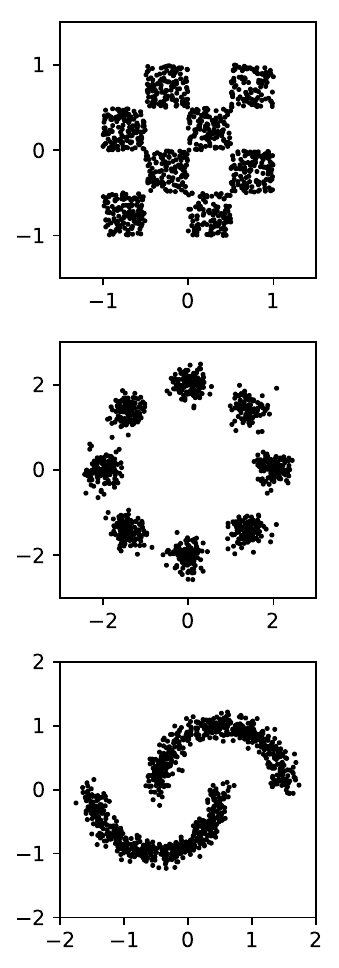}
        \caption{}
         \end{subfigure}
         \begin{subfigure}[b]{0.1695\textwidth}
         \includegraphics[width=\textwidth]{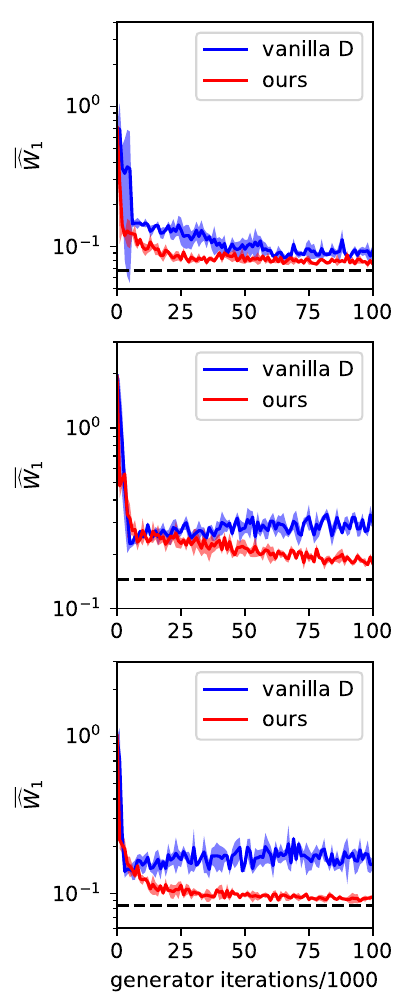}
        \caption{}
         \end{subfigure}
        \begin{subfigure}[b]{0.13\textwidth}
         \includegraphics[width=\textwidth]{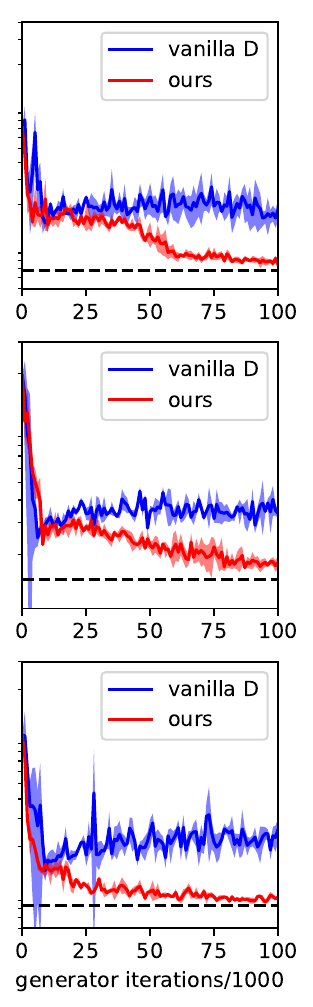}
        \caption{}
         \end{subfigure}
        \begin{subfigure}[b]{0.13\textwidth}
         \includegraphics[width=\textwidth]{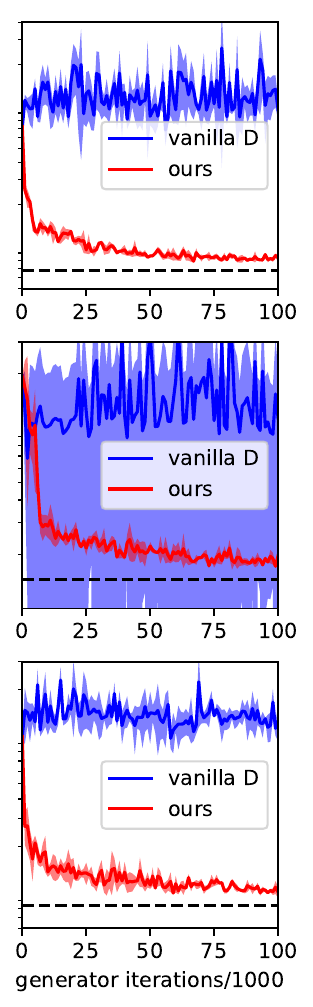}
        \caption{}
         \end{subfigure}
         \begin{subfigure}[b]{0.13\textwidth}
         \includegraphics[width=\textwidth]{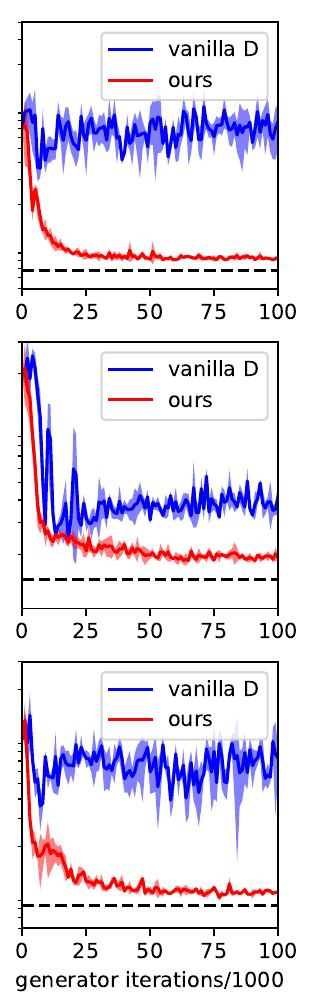}
        \caption{}
         \end{subfigure}
        \begin{subfigure}[b]{0.13\textwidth}
         \includegraphics[width=\textwidth]{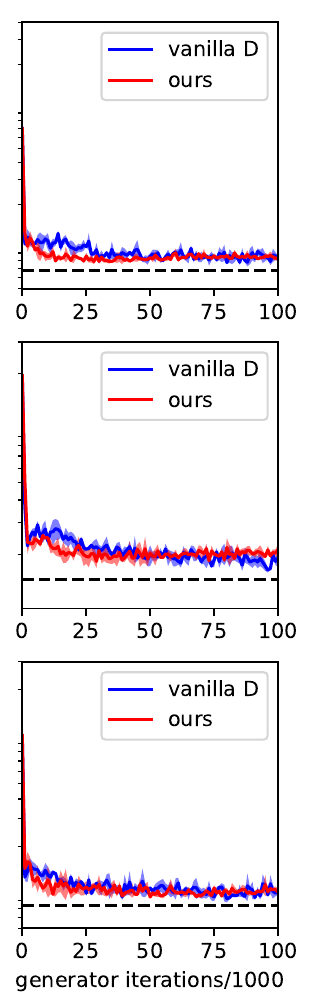}
        \caption{}
         \end{subfigure}
        \begin{subfigure}[b]{0.13\textwidth}
         \includegraphics[width=\textwidth]{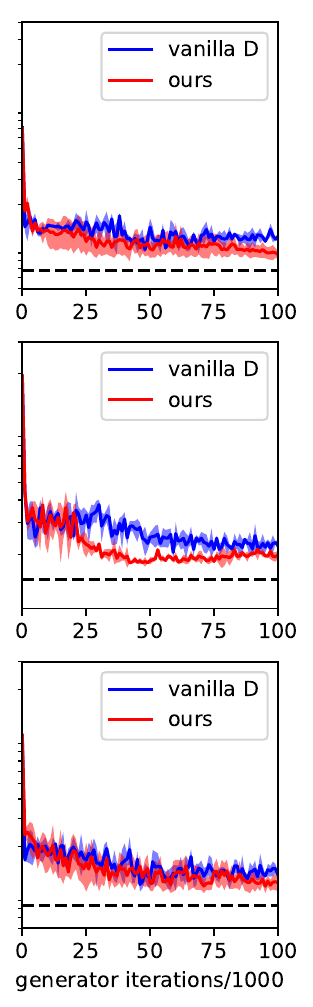}
        \caption{}
         \end{subfigure}

        \caption{Comparison between the vanilla discriminator and measure-conditional discriminator $D_{mc}$ (ours) in three 2D problems. Each row represents the results for one problem. (a): The three target distributions. (b-g): $\overline{\widehat{W}}_1(P,Q)$ against generator iterations, using different versions of GANs, discriminator/generator iteration ratios, and $(\beta_1, \beta_2)$ in Adam optimizer. (b): Vanilla GAN, 1:1, (0.5, 0.9), (c): Vanilla GAN, 1:1, (0.9, 0.999), (d): WGAN-GP, 1:1, (0.5, 0.9), (e): WGAN-GP, 1:1, (0.9, 0.999), (f): WGAN-GP, 5:1, (0.5, 0.9), (g): WGAN-GP, 5:1, (0.9, 0.999).
        The y-axes are shared for each row, and the black dashed lines represent $\overline{\widehat{W}}_1(Q,Q)$.}
        \label{fig:2D}
\end{figure*}

\section{Universal Approximation Property}\label{sec:UA}
The measure-conditional discriminators introduced above have the general form
\begin{equation}\label{eqn:abs_NN}
    \fabsnn(P_1,\dots,P_k) = g(\E_{P_1}[f_1],\dots, \E_{P_k}[f_k]),
\end{equation}
where $\E_{P_i}[f_i]$ denotes $\E_{x\sim P_i}[f_i(x)]$, each $f_j\colon \R^{n_j}\to \R^{m_j}$ for $j=1,\dots,k$ and $g\colon \R^m\to \R$ with $m:=\sum_{j=1}^km_j$ are neural networks. Note that $P_i$ can be a Dirac measure $\delta_x$, in which case $\E_{P_i}[f_i]$ is reduced to $f_i(x)$.
While \citet{pevny2019approximation} have presented a version of universal approximation theorem for nested neural networks on spaces of probability measures, the neural network architecture in Equation~\ref{eqn:abs_NN} actually takes a simpler form. We present the universal approximation theorem for the neural network in the form of Equation~\ref{eqn:abs_NN} in this Section while leaving the proof in Supplementary Material.

We use $\radonK$ to denote the space of probability distribution on a set $K$.
Let $\SglRnm$ be a space of neural networks from $\R^n$ to $\R^m$ with $l$ hidden layers and the activation $h\colon \R\to\R$, with arbitrary number of neurons in each layer. 
Let $C(\prod_{j=1}^k\radonKj;\R)$ denote the space of real-valued continuous functions on $\prod_{j=1}^k\radonKj$, which is equipped with the product of the weak topology.

\begin{thm} \label{thm:UA1}
Let $h\colon \R\to\R$ be an analytic and Lipschitz continuous non-polynomial activation function, and $K_j$ be a compact set in $\R^{n_j}$ for $j=1,\dots,k$. Let $\mathcal{H}$ be the space of functions in the form of~\ref{eqn:abs_NN} with $f_j\in \SgljRnjmj$ and $g\in \SglRm$, where $l,l_1,\dots,l_k,m,m_1,\dots, m_k\in\Z^+$.
Then, $\mathcal{H}$ is dense in $C(\prod_{j=1}^k\radonKj;\R)$ with respect to the uniform norm topology.
\end{thm}

\section{Experimental Results}\label{sec:results}
We show some results for the experimental comparisons in this section. The detailed neural network architectures are given in Supplementary Material. We emphasize that although measure-conditional discriminators $D_{mc}$ have more inputs than vanilla ones, the neural networks for both are designed to have almost the same number of parameters for the same problem. For each set-up in Section~\ref{sec:2DandImage} and \ref{sec:Stochastic} we run the code with three different random seeds; the colored lines and shaded areas in the figures represent the mean and standard deviation.

\subsection{2D Distributions and Image Generation}\label{sec:2DandImage}

We first compare the vanilla discriminator and $D_{mc}$ for different GAN setp-ups on 2D problems. In particular, we test  the vanilla GANs and WGAN-GP with different discriminator/generator iteration ratios, and $(\beta_1,\beta_2)$ in the Adam optimizer (the initial learning rate is set as 0.0001). 

\begin{figure}[H]
     \centering
    \begin{subfigure}[b]{0.23\textwidth}
    \includegraphics[width=\textwidth]{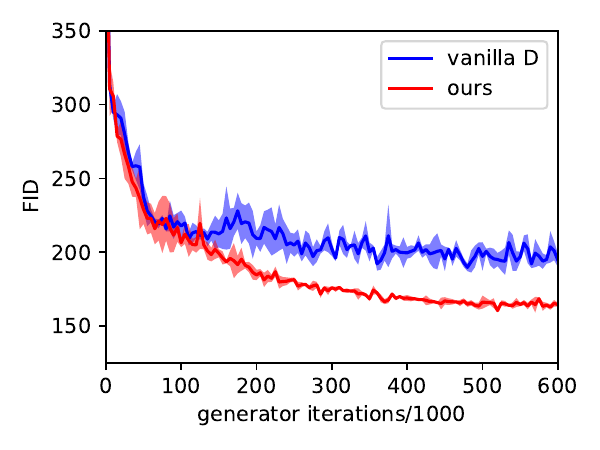}
    \caption{}
    \end{subfigure}
    \begin{subfigure}[b]{0.23\textwidth}
    \includegraphics[width=\textwidth]{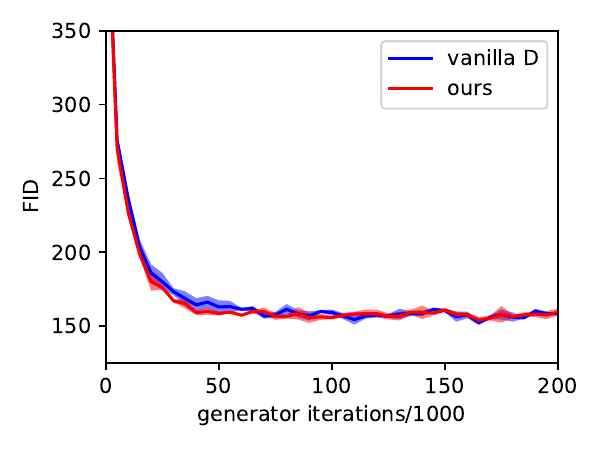}
    \caption{}
    \end{subfigure}
    \begin{subfigure}[b]{0.23\textwidth}
    \includegraphics[width=\textwidth]{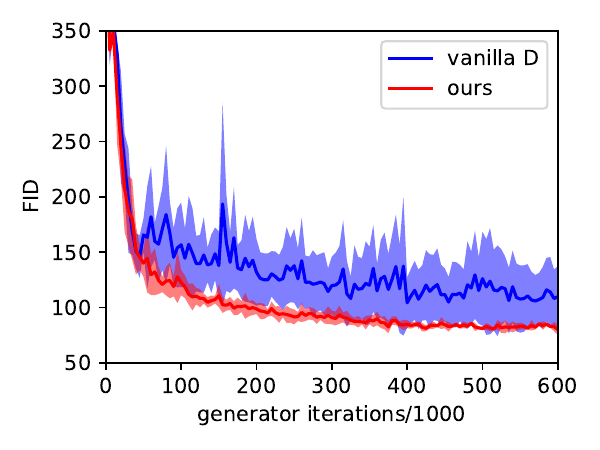}
    \caption{}
    \end{subfigure}
    \begin{subfigure}[b]{0.23\textwidth}
    \includegraphics[width=\textwidth]{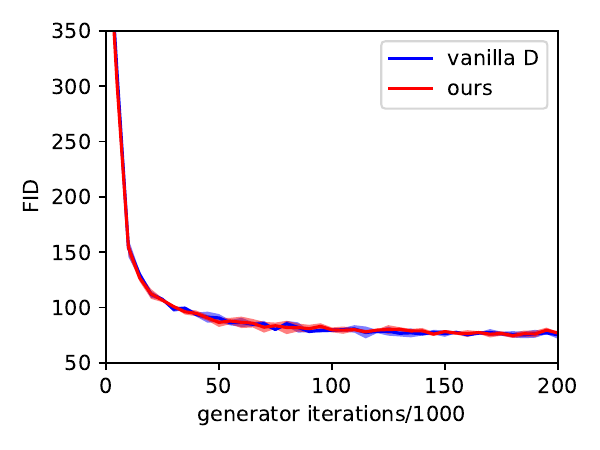}
    \caption{}
    \end{subfigure}
    \caption{FID against generator iterations in image generation tasks, with various discriminator/generator iteration ratios. (a): CIFAR10, 1:1, (b): CIFAR10, 5:1, (c): CelebA, 1:1, (d): CelebA, 5:1.}
    \label{fig:Image}
\end{figure}

To evaluate the generated distribution $P = G_{\#}\cN$, we take the expectation of empirical Wasserstein-1 distance, i.e., $\overline{\widehat{W}}_1(P,Q):= \mathbb{E}_{\hat{P}_n, \hat{Q}_n}[W_1(\hat{P}_n,\hat{Q}_n)]$
as an approximation of $W_1(P,Q)$, where $\hat{P}_n$ is the (random) empirical measure of $P$ with $n=1000$ samples, similarly for $\hat{Q}_n$. We average over 100 empirical Wasserstein distances, which can be calculated via linear programming, to calculate the expectation. The target distributions and results are shown in Figure~\ref{fig:2D} with more results in Supplementary Material. It is clear that for all set-ups except WGAN-GP with iteration ratio 5:1 and $(\beta_1,\beta_2) = (0.5,0.9)$, the measure-conditional discriminator significantly outperforms the vanilla discriminator in achieving smaller Wasserstein distances or converging faster. In fact, the measure-conditional discriminator is very robust w.r.t. the versions of GANs, the iteration ratio, and the optimizer hyperparameters, achieving approximately the same performance in different set-ups, in contrast to the vanilla discriminator.

\begin{figure*}[ht]
     \centering
    \begin{subfigure}[b]{0.19\textwidth}
    \includegraphics[width=\textwidth]{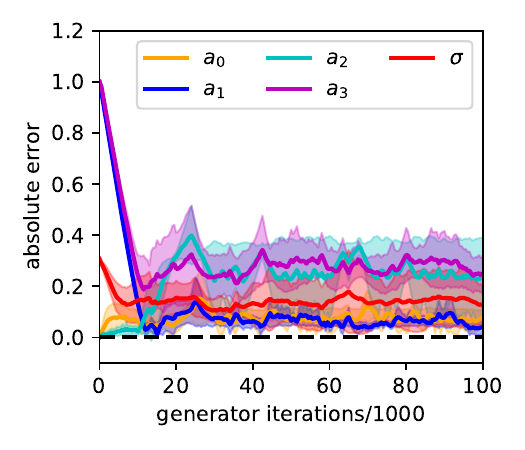}
    \end{subfigure}
    \begin{subfigure}[b]{0.19\textwidth}
    \includegraphics[width=\textwidth]{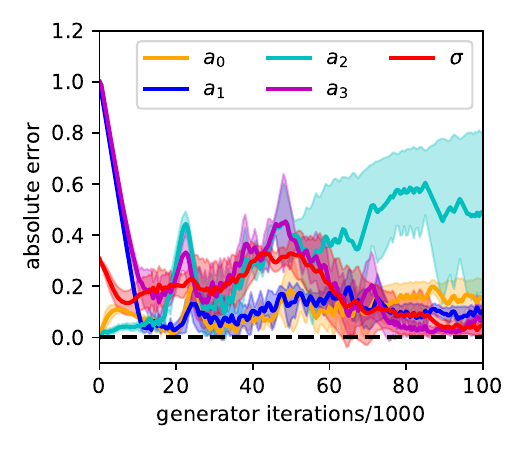}
    \end{subfigure}
    \begin{subfigure}[b]{0.19\textwidth}
    \includegraphics[width=\textwidth]{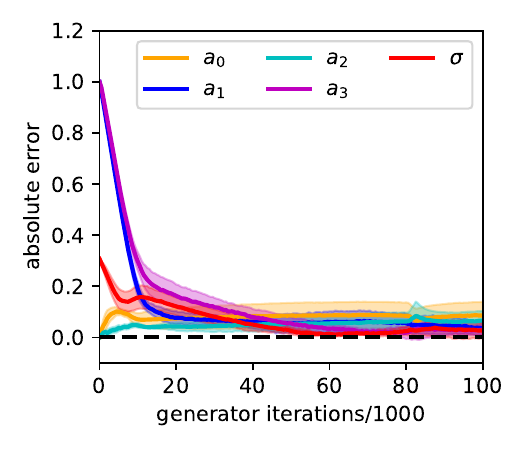}
    \end{subfigure}
    \begin{subfigure}[b]{0.19\textwidth}
    \includegraphics[width=\textwidth]{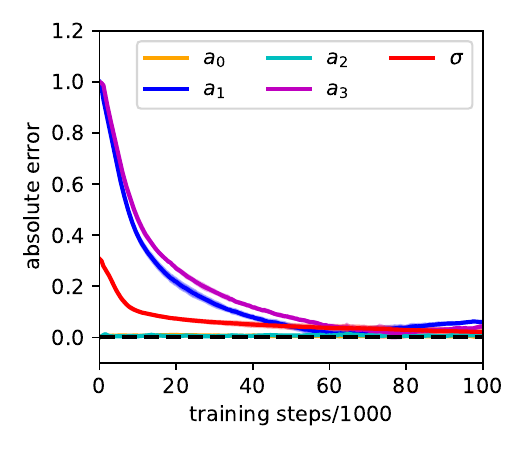}
    \end{subfigure}
    \begin{subfigure}[b]{0.19\textwidth}
    \includegraphics[width=\textwidth]{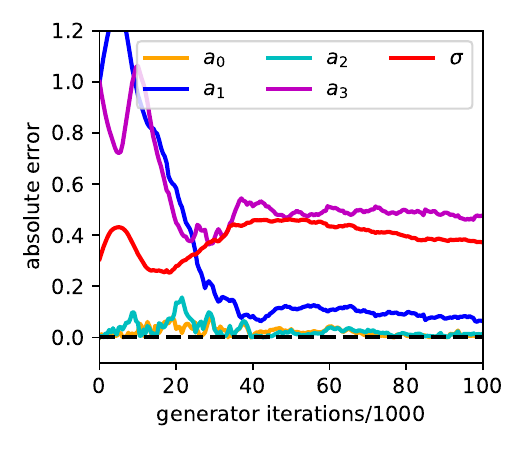}
    \end{subfigure}
    \begin{subfigure}[b]{0.19\textwidth}
    \includegraphics[width=\textwidth]{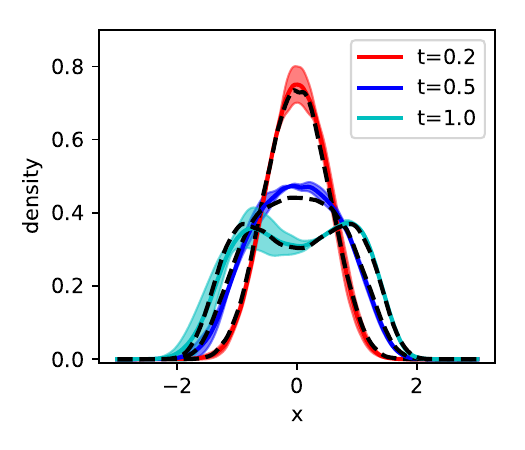}
    \caption{}
    \end{subfigure}
    \begin{subfigure}[b]{0.19\textwidth}
    \includegraphics[width=\textwidth]{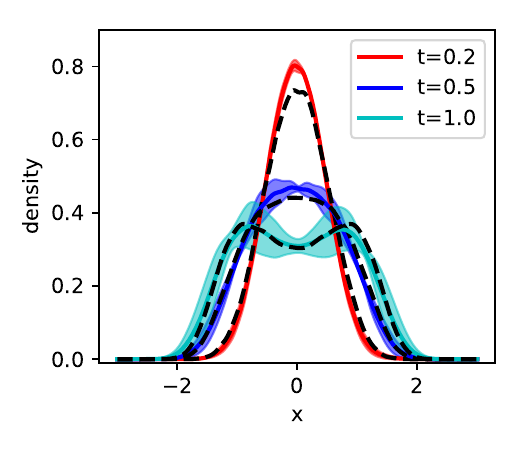}
    \caption{}
    \end{subfigure}
    \begin{subfigure}[b]{0.19\textwidth}
    \includegraphics[width=\textwidth]{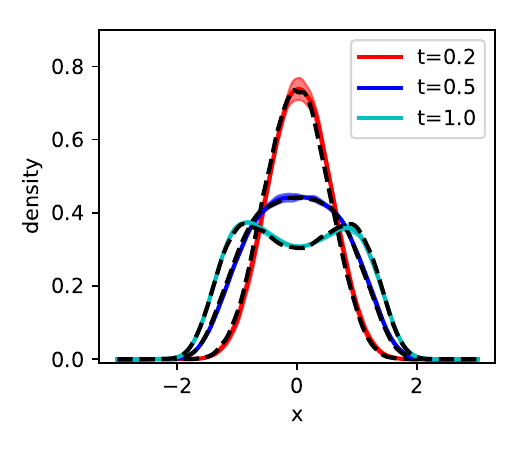}
    \caption{}
    \end{subfigure}
    \begin{subfigure}[b]{0.19\textwidth}
    \includegraphics[width=\textwidth]{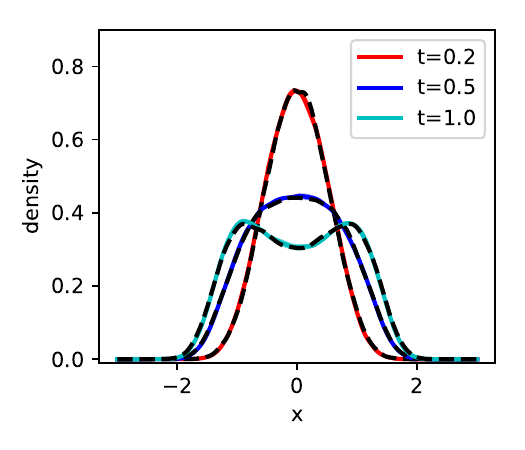}
    \caption{}
    \end{subfigure}
    \begin{subfigure}[b]{0.19\textwidth}
    \includegraphics[width=\textwidth]{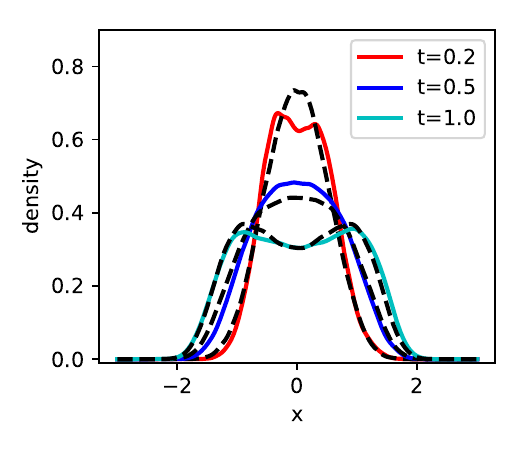}
    \caption{}
    \end{subfigure}
    
    \caption{Comparison between different set-ups in the task of stochastic dynamic inference. (a): WGAN-GP, vanilla discriminator, Adam optimizer, (b): WGAN-GP, vanilla discriminator, Optimistic Adam optimizer, (c): WGAN-GP, $D_{mc}$, Adam optimizer, (d): WGAN-GP, $D_{sr}$, Adam optimizer, (e): BGAN with Adam optimizer. First row: the absolute error of the inferred dynamic parameters $\{a_i\}_{i=0}^3$ and $\sigma$ against generator iterations. Second row: the generated distributions at $t=0.2,0.5,1.0$, in the end of the training, with the dashed black lines showing the ground truth.
    }
    \label{fig:DynamicInf}
\end{figure*}

We then compare the vanilla discriminator and measure-conditional discriminator $D_{mc}$ for image generation tasks. Specifically, we test our method on the CIFAR10 dataset~\cite{krizhevsky2009learning} and the CelebA dataset~\cite{liu2015faceattributes}, using WGAN-GP with $(\beta_1, \beta_2)$ fixed as $(0.5, 0.9)$, while two discriminator/generator iteration ratios, i.e. 1:1 and 5:1, are used. The results of Fr\'echet inception distance (FID) against the generator iterations are shown in Figure~\ref{fig:Image}.

For both tasks, while the difference between the two discriminators is negligible if the iteration ratio is set as 5:1, the measure-conditional discriminator significantly outperforms the vanilla discriminator if the iteration ratio is 1:1, achieving similar FID as in the cases with 5:1 iteration ratio. A possible explanation is that the training of the both discriminators is saturated with 5:1 iteration ratio in these two tasks. However, with 1:1 iteration ratio, the vanilla discriminator cannot give a correct guidance to the generator since it is under-trained in each iteration, while the measure-conditional discriminator can still do so by approaching its stationary target optimum in an accumulative way.

\subsection{Stochastic Dynamic Inference}\label{sec:Stochastic}
To further show the advantage of measure-conditional discriminator, here we compare it with the vanilla discriminator on the problem of inferring stochastic dynamics from observations of particle ensembles, following the framework in \citet{yang2020generative}. 
Specifically, we consider a particle system whose distributions at $t>0$, denoted as $\rho_t$, are determined by the initial distribution $\rho_0 = \cN(0,0.2)$ and the dynamics for each particle, which is governed by the stochastic ordinary differential equation:
\begin{equation}\label{eqn:dynamic}
    dx = (a_0 + a_1 x + a_2x^2 + a_3x^3)dt + \sigma dW_t,
\end{equation}
where $a_0 = 0, a_1 = 1, a_2 = 0, a_3 = -1, \sigma = 1$, and $dW_t$ is the standard Brownian motion. We consider the scenario where we do not know $\rho_0$, $\{a_i\}_{i=0}^3$ and $\sigma$, but have observations of $10^5$ indistinguishable particles at $t = 0.2, 0.5, 1.0$, which can be viewed as samples from $\rho_{0.2}$, $\rho_{0.5}$ and $\rho_{1.0}$. Our goal is to infer $\{a_i\}_{i=0}^3$ and $\sigma$ from these observations. 

Taking the standard Gaussian noise as input, the generator $G$ is a feedforward neural network whose output distribution aims to approximate $\rho_0$, followed by a first-order numerical discretization of Equation~\ref{eqn:dynamic} with $\{a_i\}_{i=0}^3$ and $\sigma$ replaced by trainable variables (the variable for $\sigma$ is activated by a softplus function to guarantee positivity), so that the particle distributions at any $t>0$ can be generated. Note that we need to tune the feedforward neural network as well as the five trainable variables to fit the target distributions $\rho_{0.2}$, $\rho_{0.5}$ and $\rho_{1.0}$ simultaneously.

We compare the following set-ups in WGAN-GP: (a) vanilla discriminators with Adam optimizer, (b) vanilla discriminators with Optimistic Adam optimizer \cite{daskalakis2018training}, which is the combination of Adam and optimistic mirror descent, (c) $D_{mc}$ in Equation~\ref{eqn:DxP} with Adam optimizer, (d) $D_{sr}$ in Equation~\ref{eqn:DxPQ} with Adam optimizer. We emphasize that the discriminator/generator iteration ratio is set as 5:1 and $(\beta_1, \beta_2) = (0.5,0.9)$, for which the measure-conditional discriminator does not outperform the vanilla one in Section~\ref{sec:2DandImage}.
We also compare with another version of GAN, i.e., (e) BGAN~\cite{lucas2018mixed}, where the discriminator also takes a distribution (instead of individual samples) as input. In short, the BGAN discriminator takes the mixture of real and generated samples as input and aims to tell the ratio of real samples.

For set-up (a,b,c,e) we have to use three discriminators, denoted as $D_{0.2}, D_{0.5},D_{1.0}$, to handle $\rho_{0.2}$, $\rho_{0.5}$, $\rho_{1.0}$, respectively. The losses for the discriminators and the generator are 
\begin{equation}
\begin{aligned}
    L_{D_{t}} &= L_{d}(G, D_{t}, \rho_{t}), t = 0.2, 0.5, 1.0 \\
    L_{G} &= \sum_{t\in S} L_{g}(G, D_{t}, \rho_{t}), S = \{0.2, 0.5, 1.0\}\\
\end{aligned}
\end{equation}
where $L_{d}(G, D_{t}, \rho_{t})$ and $L_{g}(G, D_{t}, \rho_{t})$ are discriminator and generator loss functions, given generator $G$, discriminator $D_{t}$, and a single target distribution $\rho_{t}$. We only need one $D_{sr}$ in set-up (d) since $D_{sr}$ can also take various $\rho_t$ as input. In particular, the discriminator and generator loss functions for $D_{sr}$ are  
$L_{D_{sr}} = \sum_{t\in S}  L_{d}(G, D_{sr}, \rho_{t})$ and $L_{G} = \sum_{t\in S} L_{g}(G, D_{sr}, \rho_{t})$, respectively.

In Figure~\ref{fig:DynamicInf} we visualize the results for the inferred dynamic parameters $\{a_i\}_{i=0}^3$ and $\sigma$, as well as the generated distributions at $t=0.2,0.5,1.0$ for each set-up. More results are presented in Supplementary Material. Note that $D_{mc}$ significantly outperforms the vanilla discriminator with the Adam or Optimistic Adam optimizer, even with a 5:1 iteration ratio. $D_{sr}$ achieves results as good as if not better than $D_{mc}$, and the performance is almost independent of the random seed. Moreover, since only one discriminator is involved, set-up (d) has less than half discriminator parameters compared with other setups. Such a difference in the model size will be even larger for problems with more time instants. As for BGAN, two out of three runs encountered the ``NAN'' issue, while the rest one did not outperform WGAN-GP with a measure-conditional discriminator.

\subsection{Surrogate Model for KL Divergence}

We consider the problem of approximating $D_{KL}(P||Q)$ by using the surrogate model $D_{sr}$. We set $\mu_{train}$ and $\mu_{test}$ as probability measures on $d$-dimensional Gaussian distributions, denoted as $\cN(m, \Sigma)$. We set $m_i \sim U([0.15, 0.5]\cup [-0.5,-0.15])$ for $\mu_{train}$, while $m_i \sim U([-0.15,0.15])$ for $\mu_{test}$, where the subscripts represent the component indices. For both $\mu_{train}$ and $\mu_{test}$ we set $\sqrt{\Sigma_{i,i}} \sim U([0.5,1.0])$, and the correlation coefficient $\Sigma_{i,j}/\sqrt{\Sigma_{i,i}\Sigma_{j,j}}\sim U([-0.5,0.5]) $ for $i\neq j$.

The surrogate model is trained with $P$ and $Q$ i.i.d sampled from $\mu_{train}$, with the batch size set as 100 for $(P,Q)$ pairs, and 1000 samples for each $P$ and $Q$, i.e., $n_1 = n_2 = 1000$ in Equation~\ref{eqn:DxPQ}. The surrogate model is then tested on three cases: (a): $P\sim\mu_{train},Q\sim\mu_{train}$, (b): $P\sim\mu_{train},Q\sim\mu_{test}$, (c): $P\sim\mu_{test},Q\sim\mu_{test}$. Note that there are $d(d+3)$ degrees of freedom for each $(P,Q)$ pair. For such cases there exists an analytical formula for $D_{KL}(P||Q)$ as a ground truth: 
$D_{KL}(P||Q) = \frac{1}{2}[\log\frac{|\Sigma_q|}{|\Sigma_p|} - d + tr(\Sigma_q^{-1}\Sigma_p)+ (m_p-m_q)^T\Sigma_q^{-1}(m_p-m_q)]$.

We compare the surrogate model against direct calculation via $D_{KL}(P,Q) = \E_{x\sim P}[\log(P(x)) - \log(Q(x))]$, where the densities are estimated via the kernel density estimation, for dimensionality $d=2$. In Figure~\ref{fig:KL} we quantify the accuracy against floating-point operations (FLOPs). Note that the computational cost grows linearly w.r.t. the sample size in the surrogate model, while quadratically w.r.t. the sample size in the direct calculation. The surrogate model outperforms the direct calculation in that it can achieve smaller errors with the same FLOPs for all the three $(P,Q)$ distributions in test. 
In Supplementary Material, we show scatter plots of the inference against the ground truth for dimensionality $d=2$ and 3, as well as the results of transferring the surrogate model to GANs.

\begin{figure}[ht]
     \centering
    \begin{subfigure}[b]{0.38\textwidth}
    \centering
    \includegraphics[width=1.0\textwidth]{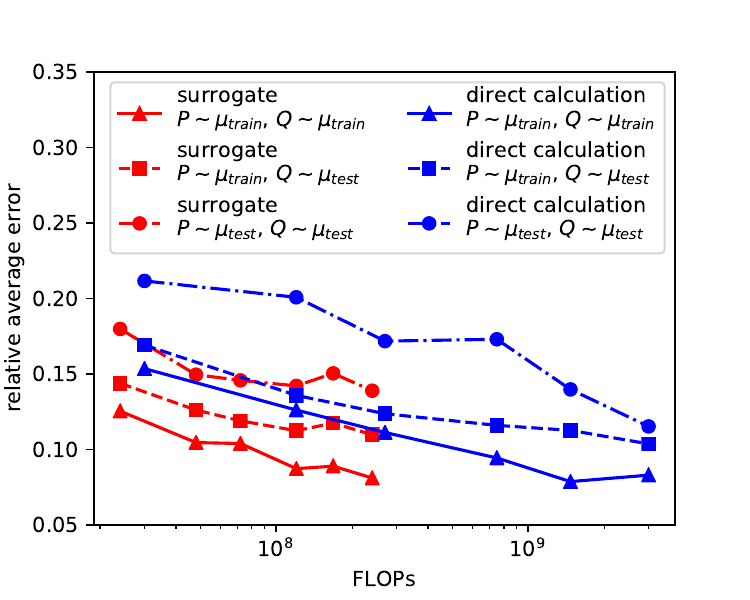}
    \end{subfigure}
    \caption{Comparison between the surrogate model and direct calculation for the KL divergence. For each line, the sample size, i.e., $n_1$ and $n_2$ in Equation~\ref{eqn:DxPQ}, increases from 1000 to 10000, to improve accuracy in the cost of FLOPs.}
    \label{fig:KL}
\end{figure}

\section{Summary and Discussion}\label{sec:Summary}
In this paper we propose measure-conditional discriminators as a plug-and-play module for a variety of GANs. Conditioned on the generated distributions so that the target optimum is stationary during training, the measure-conditional discriminators are more robust w.r.t. the GAN losses, discriminator/generator iteration ratios, and optimizer hyperparameters, compared with the vanilla ones. A variant of the measure-conditional discriminator can also be employed in the scenarios with multiple target distributions, or as surrogate models of statistical distances.

Note that even outdated generated distributions can be used to training the measure-conditional discriminator. It is worth to study if training the discriminator with generated distributions from a replay buffer, which contains the generated distributions in history just as in off-policy reinforcement learning, can further improve the performance.
Also, as a proof of concept, the neural network architectures in this paper have a very straight-forward form, leaving a lot of room for improvements. For example, different weights can be assigned to the samples of the input distributions, which is similar to importance sampling in statistics or the attention mechanism in deep learning. Moreover, the statistical distance surrogate can be applied as a building block in replacement of direct calculation in more complicated models. We leave these tasks for future research.

{\bf Acknowledgements}

We acknowledge support from the DOE PhILMs project (No. DE-SC0019453) and OSD/AFOSR MURI Grant FA9550-20-1-0358.

\newpage
\bibliographystyle{icml2021}
\bibliography{cite.bib}

\clearpage
\section{Supplementary Material}
\subsection{Neural Network Architecture}
In this section, we present the neural network architectures used in the main text. All noises into the generators are multi-variant standard Gaussians. An additional sigmoid activation is applied to the discriminator outputs in vanilla GANs. We emphasize that the vanilla discriminator and measure-conditional discriminator $D_{mc}$ share almost the same number of parameters in the same problem. In image generation tasks, the convolutional layers, denoted as ``conv'' as follows, have kernels of size $5\times5$, stride of 2, and ``same'' padding.

\textbf{2D problem}: 

generator (33,666 parameters): $2\xrightarrow{\text{dense}}128\xrightarrow{}\text{ReLU}\xrightarrow{\text{dense}}128\xrightarrow{}\text{ReLU}\xrightarrow{\text{dense}}128\xrightarrow{}\text{ReLU}\xrightarrow{\text{dense}}2$; vanilla discriminator (33,537 parameters): $2\xrightarrow{\text{dense}}128\xrightarrow{}\text{ReLU}\xrightarrow{\text{dense}}128\xrightarrow{}\text{ReLU}\xrightarrow{\text{dense}}128\xrightarrow{}\text{ReLU}\xrightarrow{\text{dense}}1$; $D_{mc}$ (33,921 parameters), $f$ and $g$: $2\xrightarrow{\text{dense}}128\xrightarrow{} \text{ReLU}\xrightarrow{\text{dense}}64$, $h$: $128\xrightarrow{\text{dense}}128\xrightarrow{} \text{ReLU}\xrightarrow{\text{dense}}1$.

\textbf{CIFAR10}: 
generator (1,565,955 parameters): $128\xrightarrow{\text{dense}}4096\xrightarrow{}\text{BatchNorm}\xrightarrow{}\text{ReLU}\xrightarrow{\text{reshape}}(4,4,256)\xrightarrow{\text{conv}}(8,8,128)\xrightarrow{}\text{BatchNorm}\xrightarrow{}\text{ReLU}\xrightarrow{\text{conv}}(16,16,64)\xrightarrow{}\text{BatchNorm}\xrightarrow{}\text{ReLU}\xrightarrow{\text{conv}}(32,32,3)\xrightarrow{}\text{tanh}$; vanilla discriminator (1,291,521 parameters): $(32,32,3)\xrightarrow{\text{conv}}(16,16,64)\xrightarrow{}\text{LeakyReLU}\xrightarrow{\text{conv}}(8,8,128)\xrightarrow{}\text{LeakyReLU}\xrightarrow{\text{conv}}(4,4,256)\xrightarrow{\text{flatten}}4096\xrightarrow{\text{dense}}64\xrightarrow{}\text{LeakyReLU}\xrightarrow{\text{dense}}1$; $D_{mc}$ (1,296,449 parameters), $f$: $(32,32,3)\xrightarrow{\text{conv}}(16,16,64)\xrightarrow{}\text{LeakyReLU}\xrightarrow{\text{conv}}(8,8,64)\xrightarrow{}\text{LeakyReLU}\xrightarrow{\text{conv}}(4,4,64)\xrightarrow{\text{flatten}}1024$, $g$: $(32,32,3)\xrightarrow{\text{conv}}(16,16,64)\xrightarrow{}\text{LeakyReLU}\xrightarrow{\text{conv}}(8,8,128)\xrightarrow{}\text{LeakyReLU}\xrightarrow{\text{conv}}(4,4,192)\xrightarrow{\text{flatten}}3072$, $h$: $4096\xrightarrow{\text{dense}}64\xrightarrow{}\text{LeakyReLU}\xrightarrow{\text{dense}}1$.

\textbf{CelebA}: 
generator (1,331,843 parameters): $128\xrightarrow{\text{dense}}8192\xrightarrow{}\text{BatchNorm}\xrightarrow{}\text{ReLU}\xrightarrow{\text{reshape}}(8,8,128)\xrightarrow{\text{conv}}(16,16,64)\xrightarrow{}\text{BatchNorm}\xrightarrow{}\text{ReLU}\xrightarrow{\text{conv}}(32,32,32)\xrightarrow{}\text{BatchNorm}\xrightarrow{}\text{ReLU}\xrightarrow{\text{conv}}(64,64,3)\xrightarrow{}\text{tanh}$; vanilla discriminator (1,307,457 parameters): $(64,64,3)\xrightarrow{\text{conv}}(32,32,32)\xrightarrow{}\text{LeakyReLU}\xrightarrow{\text{conv}}(16,16,64)\xrightarrow{}\text{LeakyReLU}\xrightarrow{\text{conv}}(8,8,128)\xrightarrow{\text{flatten}}8192\xrightarrow{\text{dense}}128\xrightarrow{}\text{LeakyReLU}\xrightarrow{\text{dense}}1$; $D_{mc}$ (1,309,921 parameters), $f$: $(32,32,3)\xrightarrow{\text{conv}}(32,32,32)\xrightarrow{}\text{LeakyReLU}\xrightarrow{\text{conv}}(16,16,32)\xrightarrow{}\text{LeakyReLU}\xrightarrow{\text{conv}}(8,8,32)\xrightarrow{\text{flatten}}2048$, $g$: $(64,64,3)\xrightarrow{\text{conv}}(32,32,32)\xrightarrow{}\text{LeakyReLU}\xrightarrow{\text{conv}}(16,16,64)\xrightarrow{}\text{LeakyReLU}\xrightarrow{\text{conv}}(8,8,96)\xrightarrow{\text{flatten}}6144$, $h$: $8192\xrightarrow{\text{dense}}128\xrightarrow{}\text{LeakyReLU}\xrightarrow{\text{dense}}1$.

\textbf{Stochastic Dynamic Inference}: 
generator for $\rho_0$ (33,409 parameters): $1\xrightarrow{\text{dense}}128\xrightarrow{}\text{tanh}\xrightarrow{\text{dense}}128\xrightarrow{}\text{tanh}\xrightarrow{\text{dense}}128\xrightarrow{}\text{tanh}\xrightarrow{\text{dense}}1$; vanilla discriminator (49,921$\times$3 parameters): $1\xrightarrow{\text{dense}}128\xrightarrow{}\text{LeakyReLU}\xrightarrow{\text{dense}}128\xrightarrow{}\text{LeakyReLU}\xrightarrow{\text{dense}}128\xrightarrow{}\text{LeakyReLU}\xrightarrow{\text{dense}}128\xrightarrow{}\text{LeakyReLU}\xrightarrow{\text{dense}}1$; $D_{mc}$ (50,177$\times$3 parameters), $f$ and $g$: $1\xrightarrow{\text{dense}}128\xrightarrow{} \text{LeakyReLU}\xrightarrow{\text{dense}}64$, $h$: $128\xrightarrow{\text{dense}}128\xrightarrow{} \text{LeakyReLU}\xrightarrow{\text{dense}}128\xrightarrow{} \text{LeakyReLU}\xrightarrow{\text{dense}}1$.  $D_{sr}$ (66,881 parameters), $f_1$, $f_2$ and $g$: $1\xrightarrow{\text{dense}}128\xrightarrow{} \text{LeakyReLU}\xrightarrow{\text{dense}}64$, $h$: $192\xrightarrow{\text{dense}}128\xrightarrow{} \text{LeakyReLU}\xrightarrow{\text{dense}}128\xrightarrow{} \text{LeakyReLU}\xrightarrow{\text{dense}}1$. The discriminator architecture in BGAN is from the original paper~\cite{lucas2018mixed}, with 4 hidden layers, each of width 128, and LeakyReLU activation, 99,329$\times$3 parameters in total.

\textbf{KL Surrogate Model}: 
$D_{sr}$ (6,137 parameters), $f_1$, $f_2$ and $g$: $2\xrightarrow{\text{dense}}32\xrightarrow{} \text{tanh}\xrightarrow{\text{dense}}32\xrightarrow{} \text{tanh}\xrightarrow{\text{dense}}8$, $h$: $24\xrightarrow{\text{dense}}32\xrightarrow{} \text{tanh}\xrightarrow{\text{dense}}32\xrightarrow{} \text{tanh}\xrightarrow{\text{dense}}1$.

\subsection{Proof of the Universal Approximation Theorem}
In this section, we provide the proof for Theorem 5.1 in the main text. Before that, we provide a useful lemma and its proof.

The following lemma provides an universal approximation theorem for the functions in the following form
\begin{equation}\label{eqn:appendix_abs_NN}
    \fabsnn(P_1,\dots,P_k) = g(\E_{P_1}[f_1],\dots, \E_{P_k}[f_k]),
\end{equation}
where $g$ is a neural network with the activation function $h$ and $l$ hidden layers, and each $f_j$ is any bounded continuous function from $\R^{n_j}$ to $\R^{m_j}$ for any positive integer $m_j$.
We denote the set containing such functions by $\SglMn$.
In the following lemma, we show that any continuous function on $\prod_{j=1}^k\tightset_j$ equipped with the product weak topology, where each $\tightset_j$ is a tight set, can be approximated using some function in $\SglMn$. 

\begin{lemma}\label{lem:uni1}
Let $h\colon \R\to\R$ be an analytic and Lipschitz continuous non-polynomial activation function. Let $k, n_1,\dots, n_k\in\Z^+$ be positive integers. Then, for each $l\in\Z^+$, $\SglMn$ is dense in $C(\prod_{j=1}^k\tightset_j;\R)$ with respect to the uniform norm topology, where $\tightset_j$ is an arbitrary tight set in $\radonj$ for each $j=1,\dots,k$.
\end{lemma}
\begin{proof}
Since any tight set in $\radonj$ is a precompact set under weak topology, then to prove the conclusion it suffices to assume $\tightset_j$ is a compact set in $\radonj$. 
Let $\tightset_j$ be any compact set in $\radonj$ for $j=1,\dots,k$. 
Let $F\colon \prod_{j=1}^k\tightset_j\to\R$ be an arbitrary continuous function w.r.t the product of the weak topology.
Let $\epsilon>0$ and $l\in\Z^+$. It suffices to prove there exist positive integers $m_1,\dots,m_k$ in $\Z^+$, a neural networks $g$ in $\SglRm$ with $m:=\sum_{j=1}^km_{j}$ and bounded continuous functions $f_1\in C_b(\R^{n_1};\R^{m_{1}}),\dots,f_k\in C_b(\R^{n_k};\R^{m_{k}})$ satisfying
\begin{equation}\label{eqt: lemma1_eq0}
    \sup_{\substack{P_j\in \tightset_j\\\forall\, j\in\{1,\dots,k\}}} |F(P_1,\dots, P_k) - g(\E_{P_1}[f_{1}],\dots, \E_{P_k}[f_{k}])| \leq \epsilon.
\end{equation}
We prove this statement by induction on $l$. We apply~\cite{Stinchcombe1999neural} in each step. 
Since $h$ is analytic, and $h$ is not a polynomial, then by Thm.2.3 in~\cite{Stinchcombe1999neural}, $S_{h,1}^1$ satisfies the assumption of Thm.5.1 in~\cite{Stinchcombe1999neural}.
First, we consider the case when $l=1$. Let $\mathcal{A}$ in Thm.5.1 in~\cite{Stinchcombe1999neural} be the vector space of measurable functions from $\prod_{j=1}^k\radonj$ to $\R$ defined by
\begin{equation*}
\begin{split}
\mathcal{A} &:= Span\{(P_1,\dots,P_k)\mapsto \E_{P_j}[f]\colon \\
&\quad\quad\quad\quad\quad\quad f\in C_b(\R^{n_j}; \R), \, j\in\{1,\dots,k\}\}\\
&= \Bigg\{(P_1,\dots,P_k)\mapsto \sum_{j=1}^k\E_{P_j}[f_j]\colon \\
&\quad\quad\quad\quad\quad\quad f_j\in C_b(\R^{n_j}; \R), \, \forall\, j\in\{1,\dots,k\}\Bigg\}.
\end{split}
\end{equation*}
Then, $\mathcal{A}$ contains any constant function, since for any constant function $f\equiv C$, we have $\E_{P_j}[f]= C$ for any $P_j\in\radonj$. Recall that each probability space on $\R^{n_j}$ is a subset of the space of Radon measures on $\R^{n_j}$, and the space of Radon measures is the dual space of $C_0(\R^{n_j};\R)$, which denotes the set of continuous functions from $\R^{n_j}$ to $\R$ which vanish at infinity.
As a result, for any distinct measures $P_j$ and $Q_j$ in $\radonj$, there exists a function $f\in C_0(\R^{n_j};\R)\subset C_b(\R^{n_j};\R)$ satisfying $\E_{P_j}[f]\neq \E_{Q_j}[f]$. Therefore, $\mathcal{A}$ separates points in $\prod_{j=1}^k\radonj$. Then, $\mathcal{A}$ satisfies the assumptions in Thm.5.1 in~\cite{Stinchcombe1999neural}, which implies that for each $\epsilon>0$, there exists a function $H$ in $Span(h\circ \mathcal{A})$ satisfying
\begin{equation}\label{eqt: lem1_limit}
    \sup_{\substack{P_j\in \tightset_j\\\forall\, j\in\{1,\dots,k\}}} |F(P_1,\dots, P_k) - H(P_1,\dots, P_k)| \leq \epsilon.
\end{equation}
Since $H$ is a function in $Span(h\circ \mathcal{A})$, then there exist a positive integer $\tilde{m}\in\Z^+$, real numbers $\alpha_1,\dots,\alpha_{\tilde{m}}\in\R$, and bounded continuous functions $\tilde{f}_{ij}\in C_b(\R^{n_j};\R)$ for each $i\in\{1,\dots, \tilde{m}\}$ and $j\in\{1,\dots,k\}$, such that there holds 
\begin{equation*}
    H(P_1,\dots,P_k) = \sum_{i=1}^{\tilde{m}} \alpha_i h\left(\sum_{j=1}^k\E_{P_j}[\tilde{f}_{ij}]\right).
\end{equation*}
Now, we prove that $H$ is a function in $\SgkMnKone$.
For each $j\in\{1,\dots, k\}$, let $f_j\colon \R^{n_j}\to\R^{\tilde{m}}$ be defined by
\begin{equation*}
    f_j(x) := (\tilde{f}_{1j}(x), \dots, \tilde{f}_{\tilde{m}j}(x)),\quad \forall x\in\R^{n_j}.
\end{equation*}
And define $g\colon \R^{k\tilde{m}}\to \R$  by
\begin{equation*}
g(x) := \sum_{i=1}^{\tilde{m}}\alpha_i h( w_i\cdot x), \quad \forall x\in\R^{k\tilde{m}},
\end{equation*}
where $w_i:=(e_i,\dots, e_i)\in \R^{k\tilde{m}}$ is a vector repeating $e_i$ for $k$ times (where $e_i$ denotes the $i$-th standard basis vector in $\R^{\tilde{m}}$).
Then, we have $f_j\in C_b(\R^{n_j};\R^{\tilde{m}})$ and $g\in \SgoneRktildem$. Moreover, after some computations, we obtain
\begin{equation} \label{eqt: lem1_H}
\begin{split}
    &g(\E_{P_1}[f_1],\dots, \E_{P_k}[f_k]) \\
    =\,& \sum_{j=1}^{\tilde{m}}\alpha_{i} h\left(w_i\cdot \left(\E_{P_1}[f_{1}], \dots,\E_{P_k}[f_{k}]\right)\right)\\
    =\,& \sum_{j=1}^{\tilde{m}}\alpha_{i} h\left(\sum_{j=1}^k e_i\cdot \E_{P_j}[f_{j}]\right)\\
    =\,& \sum_{j=1}^{\tilde{m}}\alpha_{i} h\left(\sum_{j=1}^k \E_{P_j}[e_i\cdot f_{j}]\right)\\
    =\,& \sum_{j=1}^{\tilde{m}}\alpha_{i} h\left(\sum_{j=1}^k \E_{P_j}[\tilde{f}_{ij}]\right)\\
    =\,& H(P_1,\dots, P_k).
\end{split}
\end{equation}
As a result,~\eqref{eqt: lemma1_eq0} is proved for the case of $l=1$ according to~\eqref{eqt: lem1_limit} and~\eqref{eqt: lem1_H}.

Now, assume~\eqref{eqt: lemma1_eq0} holds for some $l\in\Z^+$, and we prove the conclusion for $l+1$. Let $\mathcal{A}$ in Thm.5.1 in~\cite{Stinchcombe1999neural} be the vector space $\SglMn$. Then, $\mathcal{A}$ contains constant functions by setting $f_1,\dots,f_k$ to be constant functions in~\eqref{eqn:appendix_abs_NN}. Since $C_b(\R^{n_j};\R^{m_j})$ separates measures in $\radonj$ for each $j=1,\dots,k$ as we proved in the case of $l=1$, and the space of neural networks $\SglRm$ also separates points in $\R^m$, then the space $\mathcal{A}$ separates points in $\prod_{j=1}^k\radonj$. Therefore, $\mathcal{A}$ satisfies the assumptions in Thm.5.1 in~\cite{Stinchcombe1999neural}, which implies that
for each $\epsilon>0$, there exist $\tilde{m}\in\Z^+$, real numbers $\alpha_1,\dots, \alpha_{\tilde{m}}$ in $\R$, and functions $H_1,\dots, H_{\tilde{m}}$ in $\SglMn$ satisfying
\begin{equation}\label{eqt: lem1_eq2}
\sup_{\substack{P_j\in \tightset_j\\\forall\, j\in\{1,\dots,k\}}} \left|F(P_1,\dots, P_k) - \sum_{i=1}^{\tilde{m}}\alpha_{i}h(H_{i}(P_1,\dots,P_k))\right| \leq \epsilon.
\end{equation}
For each $i\in\{1,\dots,\tilde{m}\}$, 
since $H_{i}$ is a function in $\SglMn$, there exist positive integers $\tilde{m}_{i1},\dots, \tilde{m}_{ik}\in\Z^+$, bounded continuous functions
$\tilde{f}_{ij}\in C_b(\R^{n_j};\R^{\tilde{m}_{ij}})$ for each $j\in\{1,\dots, k\}$, and a function $\tilde{g}_{i}\in \SglRtildemi$ with $\tilde{m}_{i0}:=\sum_{j=1}^k\tilde{m}_{ij}$, 
such that $H_i(P_1,\dots, P_k) = \tilde{g}_{i}(\E_{P_1}[\tilde{f}_{i1}], \dots, \E_{P_k}[\tilde{f}_{ik}])$ holds for each $P_1\in\radonone, \dots, P_k\in\radonk$. As a result, we have
\begin{equation*}
\begin{split}
    &\sum_{i=1}^{\tilde{m}}\alpha_{i}h(H_{i}(P_1,\dots, P_k))\\
    =\,& \sum_{i=1}^{\tilde{m}}\alpha_{i}h\left(\tilde{g}_{i}\left(\E_{P_1}[\tilde{f}_{i1}], \dots, \E_{P_k}[\tilde{f}_{ik}]\right)\right),
\end{split}
\end{equation*}
for each $(P_1,\dots, P_k)\in\prod_{j=1}^k\radonj$.
Set $\tilde{m}_{0j}:= \sum_{i=1}^{\tilde{m}} \tilde{m}_{ij}$ and $m:= \sum_{j=1}^k\tilde{m}_{0j} = \sum_{i=1}^{\tilde{m}} \tilde{m}_{i0}$. For each $j\in\{1,\dots,k\}$, define $f_j\in C_b(\R^{n_j};\R^{\tilde{m}_{0j}})$ by
\begin{equation*}
   f_j(x) := \left(\tilde{f}_{1j}(x), \tilde{f}_{2j}(x), \dots, \tilde{f}_{\tilde{m}j}(x)\right),
\end{equation*}
for each $x\in\R^{n_j}$. For each $i\in\{1,\dots, \tilde{m}\}$, define $g_{i}\in \SglRm$ by
\begin{equation*}
    g_{i}(x_1, x_2,\dots, x_{k}) := \tilde{g}_{i}((x_1)_i,\dots, (x_k)_i),
\end{equation*}
for each $x_1\in\R^{\tilde{m}_{01}},\dots, x_k\in\R^{\tilde{m}_{0k}}$, where each $(x_j)_i\in\R^{\tilde{m}_{ij}}$ denotes the vector whose $r$-th component is the $\left(\sum_{I=1}^{i-1}\tilde{m}_{Ij} + r\right)$-th component of $x_j$.
With this notation, for each $i\in\{1,\dots,\tilde{m}\}$ and each $j\in\{1,\dots,k\}$, we have
\begin{equation*}
    \left(\E_{P_j}[f_j]\right)_i =\left(\E_{P_j}[\tilde{f}_{1j}],\dots, \E_{P_j}[\tilde{f}_{\tilde{m}j}]\right)_i
    =\E_{P_j}[\tilde{f}_{ij}].
\end{equation*}
Moreover, we define $g\in \SglponeRm$ by
\begin{equation*}
  \begin{split}
      g(x_1,x_2,\dots, x_{k}) &:= \sum_{i=1}^{\tilde{m}}\alpha_{i}h(g_{i}(x_1, x_2,\dots, x_{k})) \\
  & = \sum_{i=1}^{\tilde{m}}\alpha_{i}h\left(\tilde{g}_{i}((x_1)_i,\dots, (x_k)_i)\right)
  \end{split}  
\end{equation*}
for each $x_1\in\R^{\tilde{m}_{01}},\dots, x_k\in\R^{\tilde{m}_{0k}}$.
Then, after some computations, we obtain
\begin{equation*}
\begin{split}
    &g(\E_{P_1}[f_1],\dots, \E_{P_k}[f_k])\\
    =\,&
    \sum_{i=1}^{\tilde{m}}\alpha_{i}h\left(\tilde{g}_i\left(\left(\E_{P_1}[f_1]\right)_i, \dots, \left(\E_{P_k}[f_k]\right)_i\right)\right)\\
    =\,& \sum_{i=1}^{\tilde{m}}\alpha_{i}h\left(\tilde{g}_i\left(\E_{P_1}[\tilde{f}_{i1}], \dots, \E_{P_k}[\tilde{f}_{ik}]\right)\right)\\
    =\,&\sum_{i=1}^{\tilde{m}}\alpha_{i}h(H_{i}(P_1,\dots, P_k)).
\end{split}
\end{equation*}
Combining this with~\eqref{eqt: lem1_eq2}, we conclude that~\eqref{eqt: lemma1_eq0} holds for $l+1$. 
Therefore, the conclusion holds by induction.
\end{proof}

\textbf{Proof of Theorem 5.1}
Let $\epsilon>0$. It suffices to construct $m_1,\dots, m_k\in\Z^+$, $g\in \SglRm$ with $m:=\sum_{j=1}^k m_j$, and $f_j\in \SgljRnjmj$ for each $j=1,\dots, k$, such that there holds
\begin{equation} \label{eqt: pf_lem2_0}
    |F(P_1,\dots, P_k) - g(\E_{P_1}[f_1],\dots, \E_{P_k}[f_k])| \leq \epsilon,
\end{equation}
for any $(P_1,\dots, P_k)\in\prod_{j=1}^k\radonKj$.
Since each $K_j$ is a compact set in $\R^{n_j}$, 
then $\radonKj$ is tight in $\radonj$. 
Then, by Lemma~\ref{lem:uni1}, there exist $m_1,\dots, m_k\in\Z^+$, $g\in \SglRm$ with $m:=\sum_{j=1}^km_j$, and $\tilde{f}_j\in C_b(\R^{n_j};\R^{m_j})$ for each $j\in\{1,\dots, k\}$ satisfying
\begin{equation}\label{eqt: pf_lem2_1}
    \left|F(P_1,\dots, P_k) - g\left(\E_{P_1}[\tilde{f}_1],\dots, \E_{P_k}[\tilde{f}_k]\right)\right| < \frac{\epsilon}{2},
\end{equation}
for any $(P_1,\dots, P_k)\in\prod_{j=1}^k\radonKj$.
Since the activation function $h$ is Lipschitz, and the Lipschitz property is preserved under composition, then the function $g$ is also Lipschitz. Denote by $L>0$ the Lipschitz constant of $g$.
By the universal approximation theorem for neural networks (for instance, see~\cite{Kidger2020universal}), for each $j\in\{1,\dots,k\}$, there exists a neural network $f_j\in \SgljRnjmj$ satisfying
\begin{equation}\label{eqt: pf_lem2_2}
    \sup_{x\in K_j} \|f_j(x) - \tilde{f}_j(x)\| < \frac{\epsilon}{2L\sqrt{k}}.
\end{equation}

Now, we prove~\eqref{eqt: pf_lem2_0}. For each $j\in\{1,\dots, k\}$, let $P_j$ be an arbitrary measure in $\radonKj$.
Combining~\eqref{eqt: pf_lem2_1} and~\eqref{eqt: pf_lem2_2}, we obtain
\begin{equation*}
\begin{split}
    &\,|F(P_1,\dots, P_k) - g(\E_{P_1}[f_1], \dots, \E_{P_k}[f_k])|\\
    \leq&\, \left|g\left(\E_{P_1}[\tilde{f}_1],\dots, \E_{P_k}[\tilde{f}_k]\right) - g\left(\E_{P_1}[f_1],\dots, \E_{P_k}[f_k]\right)\right|\\
    &\quad +\left|F(P_1,\dots, P_k) - g\left(\E_{P_1}[\tilde{f}_1],\dots, \E_{P_k}[\tilde{f}_k]\right)\right| \\
    <&\, L\left\|\left(\E_{P_1}[\tilde{f}_1],\dots, \E_{P_k}[\tilde{f}_k]\right) - \left(\E_{P_1}[f_1],\dots, \E_{P_k}[f_k]\right)\right\| + \frac{\epsilon}{2}\\
    \leq&\,L\sqrt{k}\sup_{j\in\{1,\dots,k\}}\left\{\sup_{x\in K_j}\|\tilde{f}_j(x) - f_j(x)\|\right\} + \frac{\epsilon}{2}\\
    \leq&\, \epsilon,
\end{split}
\end{equation*}
where the second inequality holds by~\eqref{eqt: pf_lem2_1} and the Lipschitz property of $g$, the third inequality holds by the assumption that each $P_j$ is supported in the compact set $K_j$, and the fourth inequality holds according to~\eqref{eqt: pf_lem2_2}.
\qed

\subsection{More Results for 2D Problems}
In Figure~\ref{fig:2DSup} we show the comparison between the vanilla discriminator and measure-conditional discriminator $D_{mc}$ on three 2D problems, using vanilla GAN with 5:1 discriminator/generator iteration ratio. We encountered ``NAN'' issue occasionally with both discriminators, the corresponding runs are omitted. The measure-conditional discriminator outperforms the vanilla one, as in the main text.

\begin{figure}[ht]
     \centering
         \centering
        \begin{subfigure}[b]{0.147\textwidth}
         \includegraphics[width=\textwidth]{Plots/All-target.pdf}
        \caption{}
         \end{subfigure}
         \begin{subfigure}[b]{0.1695\textwidth}
         \includegraphics[width=\textwidth]{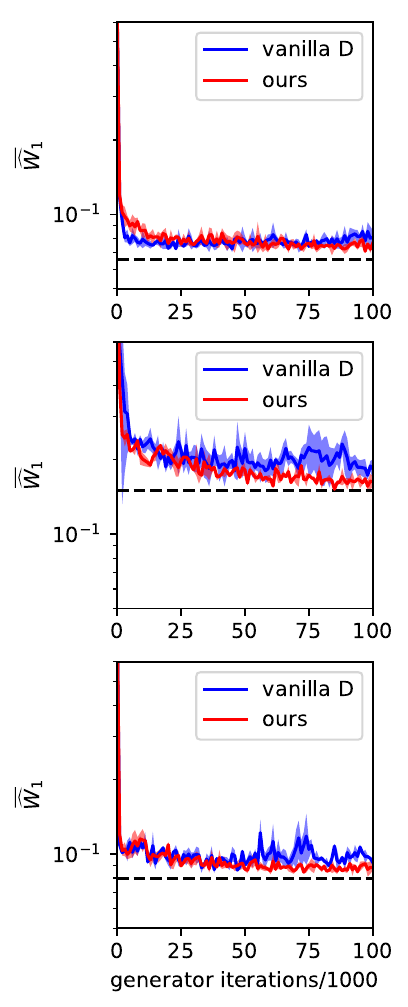}
        \caption{}
         \end{subfigure}
        \begin{subfigure}[b]{0.13\textwidth}
         \includegraphics[width=\textwidth]{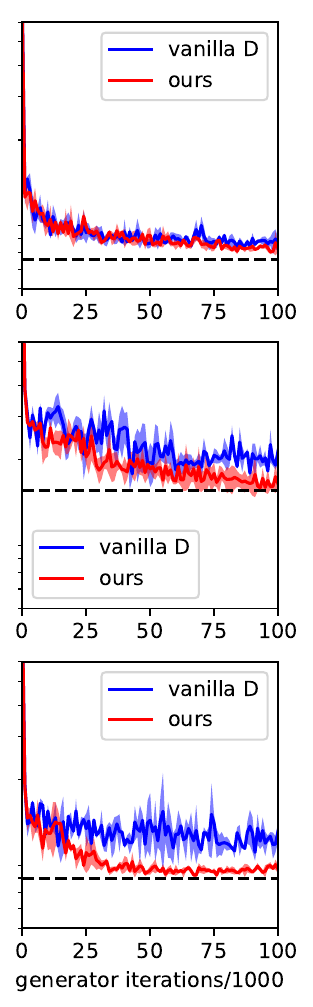}
        \caption{}
         \end{subfigure}
        \caption{More results for the comparison between the vanilla discriminator and measure-conditional discriminator (ours) in three 2D problems. (a): The three target distributions. (b): Vanilla GAN, 5:1, (0.5, 0.9), (c): Vanilla GAN, 5:1, (0.9, 0.999). See a more detailed caption in Figure 2 in the main text.}
        \label{fig:2DSup}
\end{figure}

\subsection{More Results for Stochastic Dynamic Inference}

In Figure~\ref{fig:DynamicInfMore} we show the results of $D_{mc}$ and $D_{sr}$ with the Optimistic Adam optimizer on the task of stochastic dynamic inference, using the same neural networks as in the main text. The results are similar to those with the Adam optimizer, with a slight improvement.

\begin{figure}[H]
     \centering
    \begin{subfigure}[b]{0.19\textwidth}
    \includegraphics[width=\textwidth]{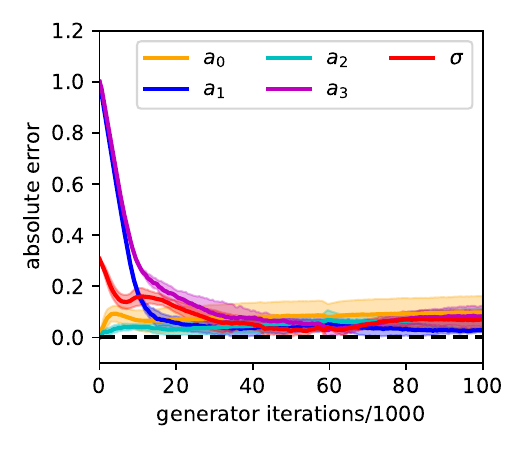}
    \end{subfigure}
    \begin{subfigure}[b]{0.19\textwidth}
    \includegraphics[width=\textwidth]{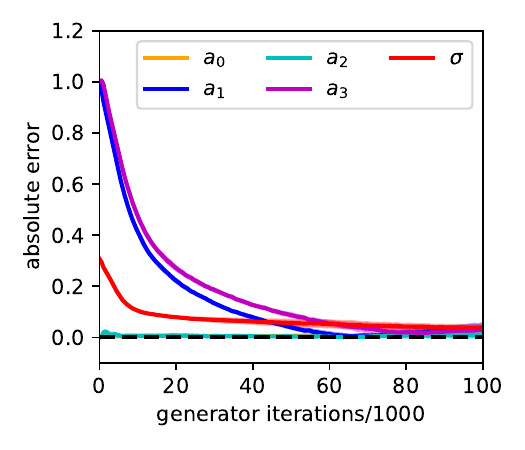}
    \end{subfigure}
    \begin{subfigure}[b]{0.19\textwidth}
    \includegraphics[width=\textwidth]{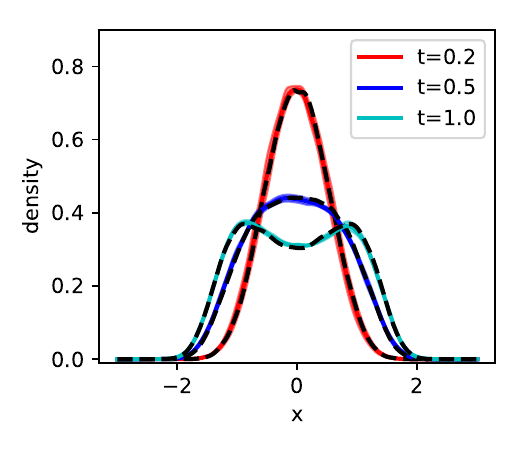}
    \caption{}
    \end{subfigure}
    \begin{subfigure}[b]{0.19\textwidth}
    \includegraphics[width=\textwidth]{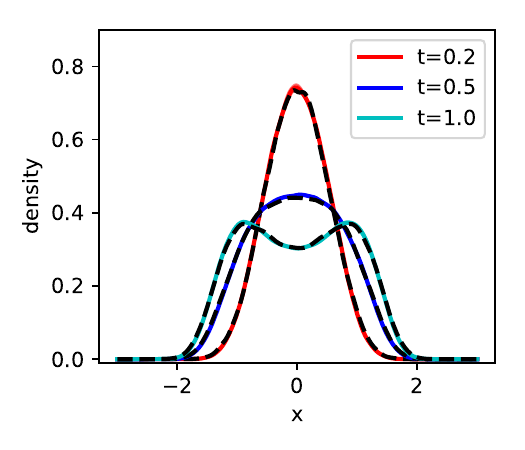}
    \caption{}
    \end{subfigure}
    \caption{Comparison between different set-ups in the task of stochastic dynamic inference. (a): WGAN-GP, $D_{mc}$, Optimistic Adam optimizer, (b): WGAN-GP, $D_{sr}$, Optimistic Adam optimizer. See a more detailed caption in Figure 4 in the main text.
    }
    \label{fig:DynamicInfMore}
\end{figure}

In addition, in Figure~\ref{fig:DynamicInfSmall} we show the results using smaller neural networks for the vanilla discriminator and $D_{mc}$ (the number of hidden layers for the vanilla discriminator and $h$ in $D_{mc}$ is reduced by 1). For the vanilla discriminator, the Optimistic Adam optimizer manages to remove the high-frequency oscillation, compared with the Adam optimizer, but the inferred parameters are still incorrect. In contrast, both optimizers give good inference with $D_{mc}$ discriminator, and the Optimistic Adam optimizer performs better in that the inference converges faster.
\begin{figure}[ht]
     \centering
    \begin{subfigure}[b]{0.19\textwidth}
    \includegraphics[width=\textwidth]{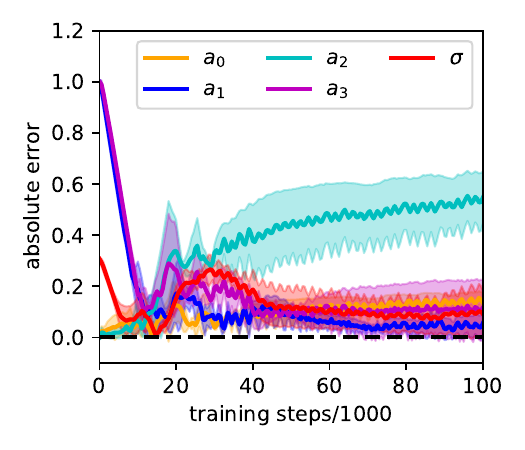}
    \end{subfigure}
    \begin{subfigure}[b]{0.19\textwidth}
    \includegraphics[width=\textwidth]{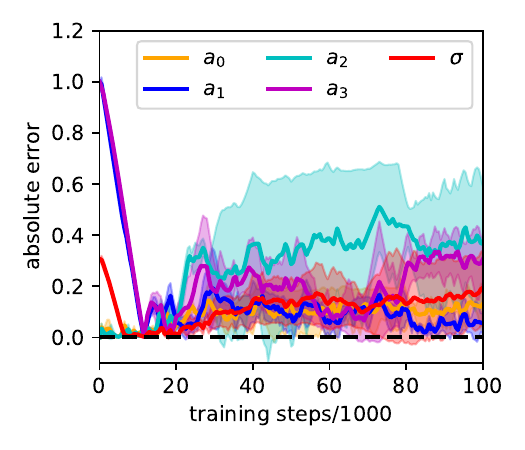}
    \end{subfigure}
    \begin{subfigure}[b]{0.19\textwidth}
    \includegraphics[width=\textwidth]{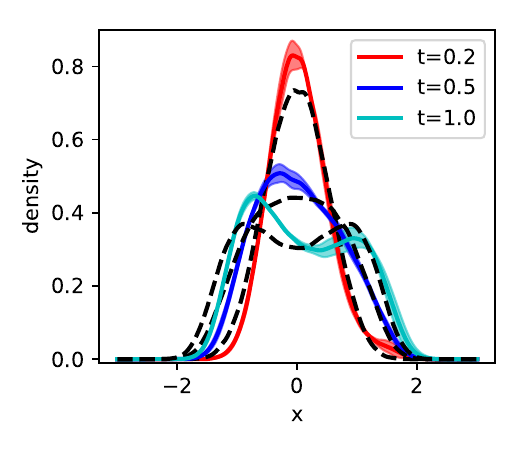}
    \caption{}
    \end{subfigure}
    \begin{subfigure}[b]{0.19\textwidth}
    \includegraphics[width=\textwidth]{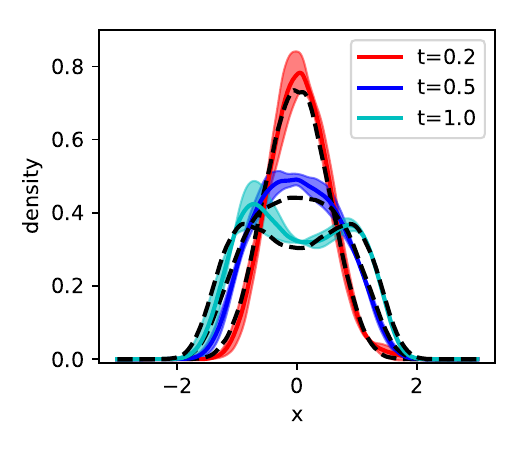}
    \caption{}
    \end{subfigure}
        \begin{subfigure}[b]{0.19\textwidth}
    \includegraphics[width=\textwidth]{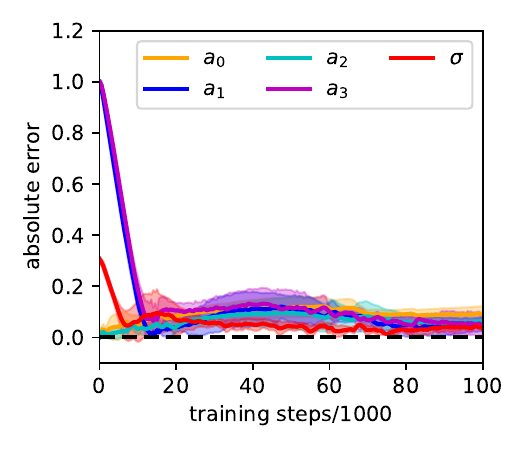}
    \end{subfigure}
    \begin{subfigure}[b]{0.19\textwidth}
    \includegraphics[width=\textwidth]{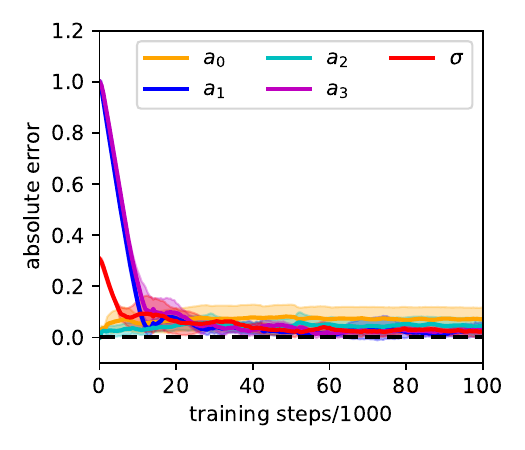}
    \end{subfigure}
    \begin{subfigure}[b]{0.19\textwidth}
    \includegraphics[width=\textwidth]{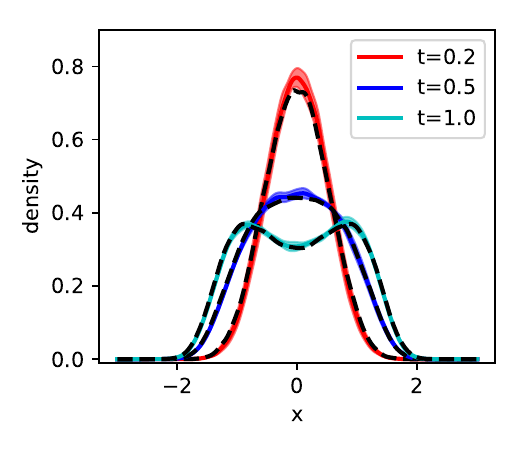}
    \caption{}
    \end{subfigure}
    \begin{subfigure}[b]{0.19\textwidth}
    \includegraphics[width=\textwidth]{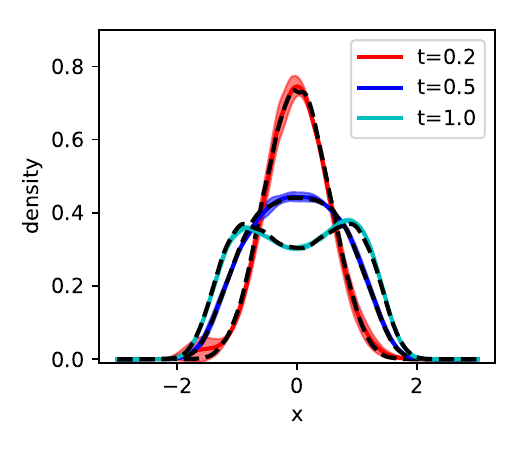}
    \caption{}
    \end{subfigure}
    \caption{Comparison between different set-ups in the task of stochastic dynamic inference, using smaller discriminator neural networks. (a): WGAN-GP, vanilla discriminator, Adam optimizer, (b): WGAN-GP, vanilla discriminator, Optimistic Adam optimizer, (c): WGAN-GP, $D_{mc}$, Adam optimizer, (d): WGAN-GP, $D_{mc}$, Optimistic Adam optimizer. See a more detailed caption in Figure 4 in the main text.
    }
    \label{fig:DynamicInfSmall}
\end{figure}

\subsection{More Results for the Statistical Distance Surrogate}
As a supplement of Figure 5 in the main text, in Figure~\ref{fig:KLscatter1},~\ref{fig:KLscatter2} we show the scatter plots of the inference of KL divergence against the ground truth for dimensionality $d=2$. In Figure~\ref{fig:KLscatter3},~\ref{fig:KLscatter4} we also show the results for the 3D case, with a larger $D_{sr}$ neural network (128 as the hidden layer width and 32 as the output dimension of $f_1$, $f_2$ and $g$).

\begin{figure}[H]
     \centering
    \begin{subfigure}[b]{0.23\textwidth}
    \includegraphics[width=\textwidth]{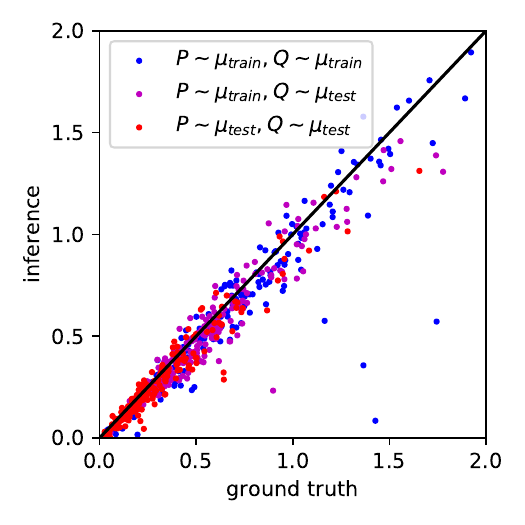}
    \caption{}
    \label{fig:KLscatter1}
    \end{subfigure}
    \begin{subfigure}[b]{0.23\textwidth}
    \includegraphics[width=\textwidth]{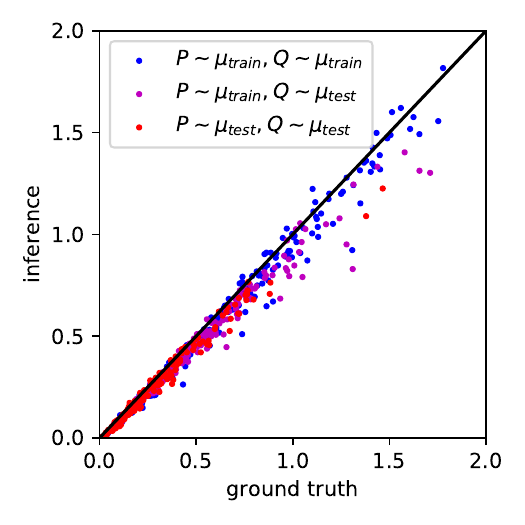}
    \caption{}
    \label{fig:KLscatter2}
    \end{subfigure}
    
    \begin{subfigure}[b]{0.23\textwidth}
    \includegraphics[width=\textwidth]{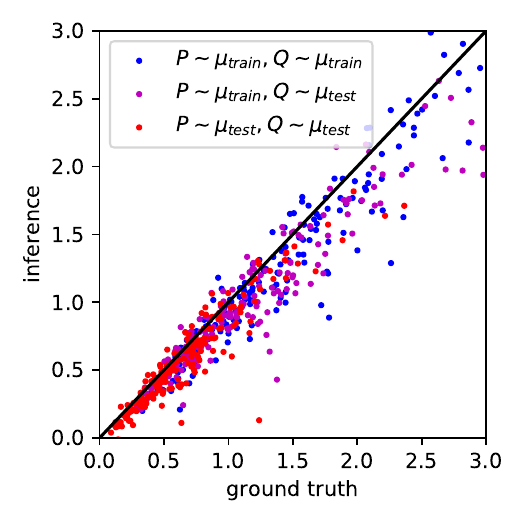}
    \caption{}
    \label{fig:KLscatter3}
    \end{subfigure}
    \begin{subfigure}[b]{0.23\textwidth}
    \includegraphics[width=\textwidth]{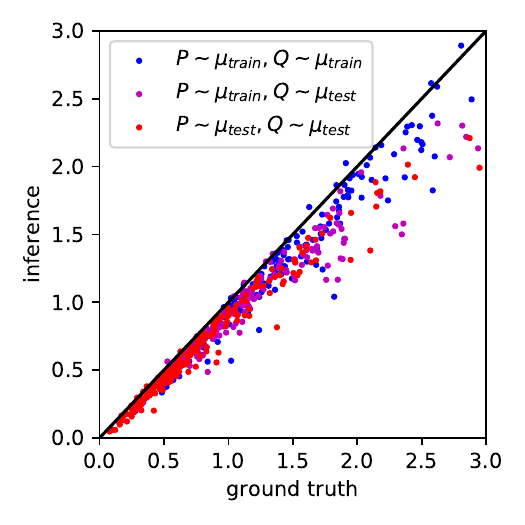}
    \caption{}
    \label{fig:KLscatter4}
    \end{subfigure}

    \begin{subfigure}[b]{0.23\textwidth}
    \includegraphics[width=\textwidth]{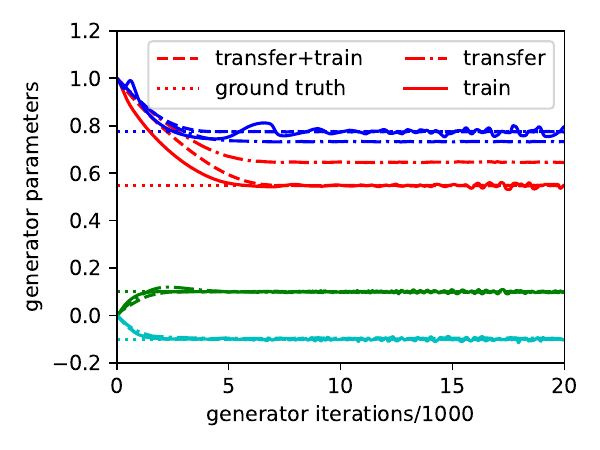}
    \caption{}
    \label{fig:KLscatter5}
    \end{subfigure}
    \begin{subfigure}[b]{0.23\textwidth}
    \includegraphics[width=\textwidth]{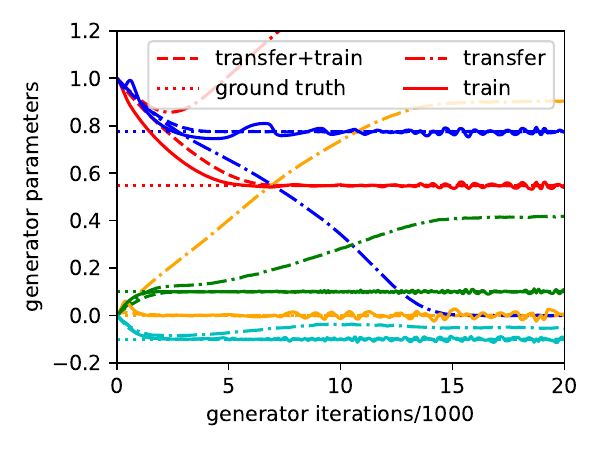}
    \caption{}
    \label{fig:KLscatter6}
    \end{subfigure}
    
    \caption{Results for the KL divergence surrogate model. (a-d): Inference against the ground truth. (a): 2D, 1000 samples, (b): 2D, 10000 samples, (c): 3D, 1000 samples, (d): 3D, 10000 samples. (e-f) Generator parameters during the GAN training with different generator and discriminator set-ups. (e): The first generator set-up, with 4 degrees of freedom. (f): The second generator set-up, with 5 degrees of freedom. Different colors represent different generator parameters, while different line styles represent the results from different discriminator setups and the ground truth.}
    \label{fig:KLscatter}
\end{figure}

As a proof of concept, we then employ the 2D surrogate model as a discriminator in GAN. The target distribution is set as $\cN(m, \Sigma)$ with $m = [0.1,-0.1]$ and $\Sigma = \text{diag}([0.3, 0.6])$, which is a sample from $\mu_{test}$. The generator is defined as $G(z) = Az + b$, and we test with two generator set-ups: (1) $A = \text{diag}([a_1, a_2]), b = [b_1, b_2]$ with 4 degrees of freedom, and (2) $A = [[a_1, a_2],[0, a_3]], b = [b_1,b_2]$  with 5 degrees of freedom, both having ground truth for the parameters. We compare the following three set-ups of the discriminator: (a) $D_{sr}$ transferred from the well-trained surrogate model and is further trained in GAN, (b) transferred $D_{sr}$ without further training in GAN, (c) random initialized $D_{sr}$ with training in GAN. The generator parameters during the GAN training are visualized in Figure~\ref{fig:KLscatter5},~\ref{fig:KLscatter6}. One can see that for discriminator set-up (b), the generator parameters are not too bad in the first generator set-up with 4 degrees of freedom, but totally failed in the second generator setup. A possible explanation is that $G_{\#}\cN$ becomes an outlier of $\mu_{train}$ during the training and thus $D_{sr}$ cannot provide correct statistical distances. Discriminator set-ups (a) and (c) worked well in both generator setups, but note that the set-up (a), i.e. the one with transfer learning, converges faster and does not have the burrs on the curve. This demonstrates the benefit of the transfer learning with the pretrained $D_{sr}$.

\subsection{Surrogate Model for Optimal Transport Map}
In \citet{seguy2018large} the authors proposed a two-step method for learning the barycentric projections of regularized optimal transport, as approximations of optimal transport maps between continuous measures. Their method solves the map between one pair of measures in one training process, but we can make a modification with measure-conditional discriminators and obtain a surrogate model for the optimal transport maps between various pairs of measures.

\begin{figure}[ht]
     \centering
    \begin{subfigure}[b]{0.45\textwidth}
    \includegraphics[width=\textwidth]{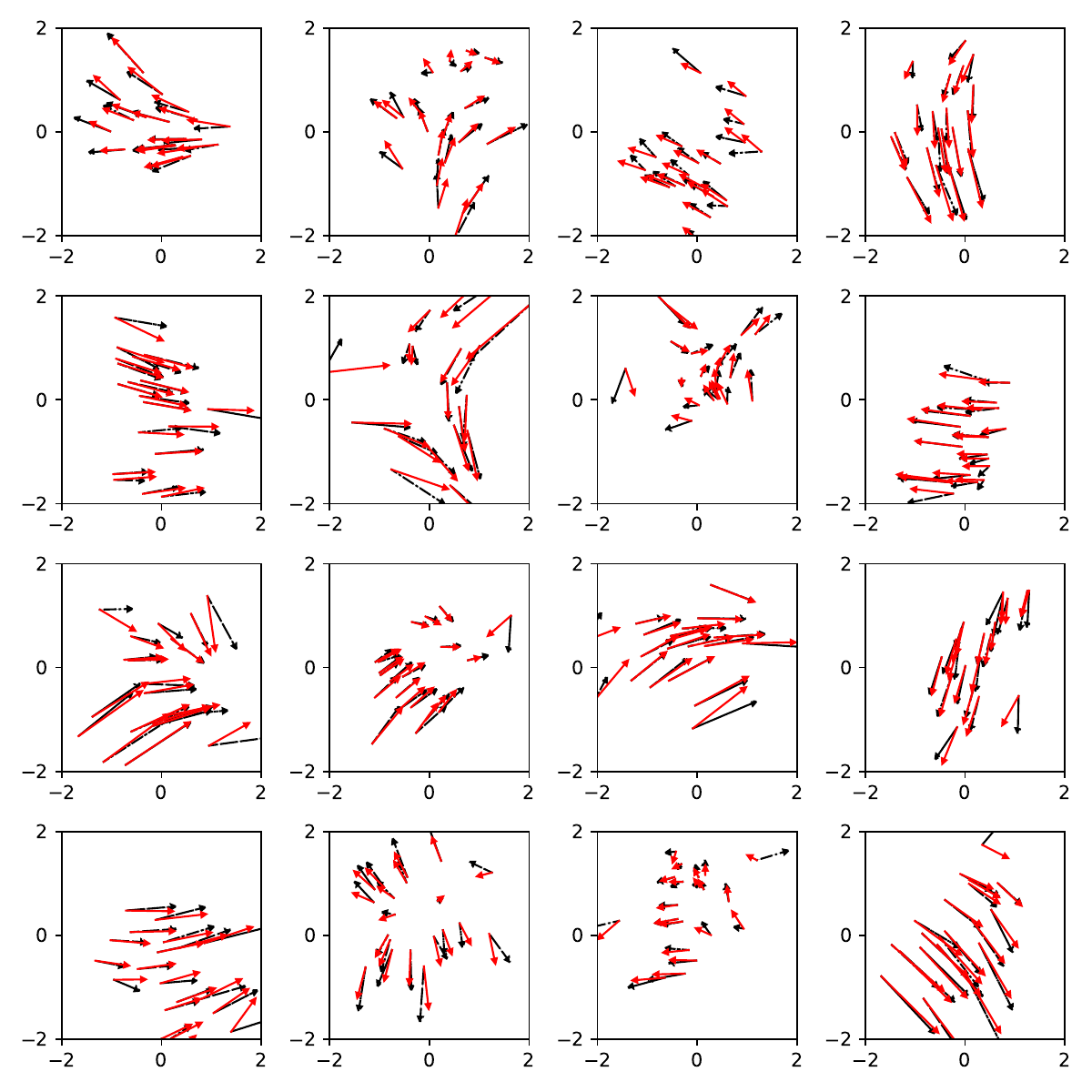}
    \end{subfigure}
    \caption{Results of the surrogate model for optimal transport maps between 16 pairs of Gaussian distributions. The red arrows represent the barycentric projection maps given by the surrogate model $D_{sr,G}$, while the black arrows represent the reference optimal transport maps from linear programming.}
    \label{fig:RegOT}
\end{figure}

Specifically, we use two $D_{sr}$ neural networks, denoted as $D_{sr,G}$ and $D_{sr,D}$ to approximate the transport map and an auxiliary function, respectively. The first step in \citet{seguy2018large} is to maximize
\begin{equation}
    \E_{x\sim P, y\sim Q} [u(x)+v(y)-\frac{1}{4\epsilon}(u(x)+v(y)-c(x,y))_{+}^2],
\end{equation}
which is the variational form of the optimal transport cost with $L^2$ regularization, where $u$ and $v$ are two neural network to train, $\epsilon = 0.02$ is the regularization weight, and $c(x,y)$ is set as $||x-y||^2$. Utilizing the symmetry between the optimal $u$ and $v$ if we swap $P$ and $Q$, we use $D_{sr,D}(P, Q, x)$ and $D_{sr,D}(Q,P,y)$ to replace $u(x)$ and $v(y)$, respectively. The loss function for $D_{sr,D}$ writes as
\begin{equation}\label{eqn:reg_LD}
\begin{aligned}
    L_D = &\E_{(P,Q)\sim\mu}\E_{x\sim P, y\sim Q}\\ &[-D_{sr,D}(P, Q, x)-D_{sr,D}(Q,P,y)\\+&\frac{1}{4\epsilon}(D_{sr,D}(P, Q, x)+D_{sr,D}(Q,P,y)-c(x,y))_{+}^2].
\end{aligned}
\end{equation}
The second step in \citet{seguy2018large} is to train $f$ to minimize
\begin{equation}
    \E_{x\sim P, y\sim Q} [\frac{1}{2\epsilon}c(y, f(x))(u(x)+v(y)-c(x,y))_{+}],
\end{equation}
so that the minimizer $f^*$ is the barycentric projection of the regularized optimal transport, which can be viewed as an approximation of the optimal transport map from $P$ to $Q$. We will use $D_{sr,G}(P,Q,x)$ to replace $f(x)$, and the loss function for $D_{sr,G}$ writes as
\begin{equation}\label{eqn:reg_LG}
\begin{aligned}
    L_G =& \E_{(P,Q)\sim\mu}\E_{x\sim P, y\sim Q} [\frac{1}{2\epsilon}c(y, D_{sr,G}(P,Q,x))\\&(D_{sr,D}(P, Q, x)+D_{sr,D}(Q,P,y)-c(x,y))_{+}],
\end{aligned}
\end{equation}
Note that we take the expectation over $(P,Q)$ in Equation~\ref{eqn:reg_LD} and \ref{eqn:reg_LG}, so that in the end of training, $D_{sr,G}(P,Q,x)$ will approximate the optimal transport map from $P$ to $Q$ for various $(P,Q)$ pairs. 

\citet{seguy2018large} propose to train $u$ and $v$ until convergence, and then train $f$. But we found that training $D_{sr,D}$ and $D_{sr,G}$ iteratively after a warming-up training of $D_{sr,D}$ also works. We train and test with $P,Q$ independently sampled from the 2D $\mu_{train}$ in Section 6.3 of the main text, and show the results after 200,000 iterations with 10,000 warming-up steps in Figure~\ref{fig:RegOT}. The reference map is the empirical optimal transport map between 1000 samples, calculated by the POT package \cite{flamary2017pot} using linear programming. One can see that the surrogate model provides a similar transport map as the reference.

\end{document}